\documentclass[english]{llncs}
\usepackage[T1]{fontenc}
\usepackage[latin9]{inputenc}
\usepackage{array}
\usepackage{float}
\usepackage{fancybox}
\usepackage{calc}
\usepackage{mathtools}
\usepackage{multirow}
\usepackage{amsmath}
\usepackage{amssymb}
\usepackage{stmaryrd}
\usepackage{graphicx}
\usepackage{url}
\makeatletter

\providecommand{\tabularnewline}{\\}
\floatstyle{ruled}
\newfloat{algorithm}{tbp}{loa}
\providecommand{\algorithmname}{Algorithm}
\floatname{algorithm}{\protect\algorithmname}

\makeatother

\usepackage{babel}
\newcommand{\card}[1]{\ensuremath{\sharp #1}}
\begin{document}

\title{\Large Handling Nominals and Inverse Roles using Algebraic Reasoning}

\author{Humaira Farid \and Volker Haarslev}

\institute{Concordia University, Montreal \\
\email h\_arid@encs.concordia.ca, haarslev@cse.concordia.ca \vspace{-4mm}}
\maketitle

\begin{abstract}
This paper presents a novel $\mathcal{SHOI}$ tableau calculus which incorporates algebraic reasoning for deciding ontology consistency. Numerical restrictions imposed by nominals, existential and universal restrictions are encoded into a set of linear inequalities. Column generation and branch-and-price algorithms are used to solve these inequalities. Our preliminary experiments indicate that this calculus performs better on $\mathcal{SHOI}$ ontologies than standard tableau methods. 
\end{abstract}

\section{Introduction and Motivation}

Description Logic (DL) is a formal knowledge representation language
that is used for modeling ontologies. Modern description logic systems
provide reasoning services that can automatically infer implicit knowledge
from explicitly expressed knowledge. Designing reasoning algorithms
with high performance has been one of the main concerns of DL researchers.
One of the key features of many description logics is support for
nominals. Nominals are special concept names that must be interpreted
as singleton sets. They allow to use Abox individuals within concept descriptions. However, nominals carry implicit global numerical
restrictions that increase reasoning complexity. Moreover, the
interaction between nominals and inverse roles leads to the 
loss of the tree model property. Most state-of-the-art reasoners,
such as Konclude \cite{steigmiller2014konclude}, Fact++ \cite{tsarkov2006fact++},
HermiT \cite{shearer2008hermit}, have implemented traditional tableau algorithms. Konclude also incorporated consequence-based reasoning into its tableau calculus  \cite{steigmiller2014coupling}.
These reasoners try to construct completion graphs in
a highly non-deterministic way in order to handle nominals. For example, a small $\mathcal{ALCO}$ ontology models Canada consisting of its ten provinces: $\mathit{CA\_Province}\equiv\{\mathit{Ontario}$, $\mathit{Quebec}$, $\mathit{NovaScotia}$, $\mathit{NewBrunswick}$, $\mathit{Manitoba}$, $\mathit{BritishColumbia}$, $\mathit{PrinceEdwardIsland}$, $\mathit{Saskatchewan}$, $\mathit{NewfoundlandAndLabrador}$, $\mathit{Alberta}\}$.
If one tries to model that Canada consists of 11 provinces, it is trivial
to see that it is not possible because the cardinality of $\mathit{CA\_Province}$ is implicitly restricted to the 10 provinces listed as nominals. However, according to our preliminary experiments, above mentioned DL reasoners are unable to decide this inconsistency within
a reasonable amount of time. Consequence-based (CB) reasoning algorithms
are also extended to more expressive DLs such as
$\mathcal{SHOI}$ \cite{cucala2017consequence} and $\mathcal{SROI}Q$
\cite{cucala2018consequence}. Since their implementations are not available, we could not analyze these reasoners.

However, algebraic DL reasoners are considered more efficient in handling numerical restrictions \cite{faddoul2010algebraic,farsiniamarj2010practical,haarslev2001racer,vlasenko2016saturation}. RacerPro \cite{haarslev2001racer}
was the first highly optimized reasoner that combined tableau-based
reasoning with algebraic reasoning \cite{haarslev2001combining}. 
Other tableau-based algebraic reasoner for $\mathcal{SHQ}$
\cite{farsiniamarj2010practical}, $\mathcal{SHIQ}$ \cite{pour2012algebraic}, $\mathcal{SHOQ}$ \cite{faddoul2010optimizing,faddoul2010algebraic} are also proposed to handle qualified number restrictions (QNRs) and their interaction with inverse roles or nominals. These reasoners use an atomic decomposition technique to encode number restrictions into a set of linear inequalities. These inequalities are then solved by integer linear programming (ILP). These reasoners perform very efficiently in handling huge values in number restrictions. However, their ILP algorithms are best-case exponential to the number of inequalities. For example, in case of $m$ inequalities they require $2^{m}$ variables in order to find the optimal solution. However, for ILP with a huge number of variables it is not feasible to enumerate all variables. To overcome this problem, the column generation technique has been used  \cite{vlasenko2016saturation,karahroodi2017consequence} which considers a small subset of variables.
However, to the best of our knowledge, no algebraic
calculus can handle DLs supporting nominals and inverse roles simultaneously.

In this paper, we present a novel algebraic tableau calculus for $\mathcal{SHOI}$
to handle a large number of nominals and their interaction with
inverse roles. The rest of this paper is structured as follows. Section
2 defines important terms and introduces $\mathcal{SHOI}$.
Section 3 presents the algebraic tableau calculus for $\mathcal{SHOI}$.
Section 4 provides evaluation results for the implemented prototype Cicada. The last section concludes our paper.

\section{Preliminaries } 

In this section, we introduce $\mathcal{SHOI}$ and some notations
used later. Let $N=N_{C}\cup N_{o}$ where $N_{C}$ represents concept names and $N_{o}$  nominals. Let $N_{R}$
be a set of role names with a set of transitive
roles $N_{R_{+}}\subseteq N_{R}$. The set of roles in $\mathcal{SHOI}$ is $N_{R}\cup\{R^{-}\mid R\in N_{R}\}$
where $R^{-}$ is called the inverse of $R$. A function $\mathsf{Inv}$
returns the inverse of a role such that $\mathsf{Inv}(R)=R^{-}$ if
$R\in N_{R}$ and $\mathsf{Inv}(R)=S$ if $R=S^{-}$ and $S\in N_{R}$. An interpretation
$\mathcal{I}=(\Delta^{\mathcal{I}},\cdot^{\mathcal{I}})$ consists
of a non-empty set $\Delta^{\mathcal{I}}$ of individuals called the
domain of interpretation and an interpretation function $\cdot^{\mathcal{I}}$.
Table~\ref{QCR-syntax} presents syntax and semantic of $\mathcal{SHOI}$. 
We use $\top$ ($\bot$) as an abbreviation for $A\sqcup\neg A$
($A\sqcap\neg A$) for some $A\in N_{C}$. 
In the following $\card{\{.\}}$ denotes set cardinality.

\begin{table}[t]
\caption{Syntax and semantics of $\mathcal{SHOI}$}
\label{QCR-syntax}
\vspace{-2mm}
\centering{}%
\begin{tabular}{|l|l|l|}
\hline 
\textbf{\small{}Construct}{\small{} } & \textbf{\small{}Syntax}{\small{} } & \textbf{\small{}Semantics}\tabularnewline
\hline 
\hline 
atomic concept  & $A$  & $A^{\mathcal{I}}\subseteq\Delta^{\mathcal{I}}$\tabularnewline
negation  & $\neg C$  & $\Delta^{\mathcal{I}}\setminus C^{\mathcal{I}}$\tabularnewline
conjunction  & $C\sqcap D$  & $C^{\mathcal{I}}\cap D^{\mathcal{I}}$\tabularnewline
disjunction  & {\small{}$C\sqcup D$}  & {\small{}$C^{\mathcal{I}}\cup D^{\mathcal{I}}$}\tabularnewline
value restriction  & $\forall R.C$  & $\{x\in\Delta^{\mathcal{I}}\mid\forall y:(x,y)\in R^{\mathcal{I}}\Rightarrow y\in C^{\mathcal{I}}\}$\tabularnewline
exists restriction  & $\exists R.C$  & $\{x\in\Delta^{\mathcal{I}}\mid\exists y\in\Delta^{\mathcal{I}}:(x,y)\in R^{\mathcal{I}}\wedge y\in C^{\mathcal{I}}\}$\tabularnewline
\hline 
nominal  & $\{o\}$  & $\card{\{o\}^{\mathcal{I}}}=1$ \tabularnewline
\hline 
role  & $R$  & $R^{\mathcal{I}}\subseteq\Delta^{\mathcal{I}}\times\Delta^{\mathcal{I}}$\tabularnewline
inverse role  & $R^{-}$  & $(R^{-})^{\mathcal{I}}=\{(y,x)\mid(x,y)\in R^{\mathcal{I}}\}$\tabularnewline
transitive role  & $R\in N_{R_{+}}$  & $(x,y)\in R^{\mathcal{I}}\wedge (y,z)\in R^{\mathcal{I}}\Rightarrow (x,z)\in R^{\mathcal{I}}$ 
\tabularnewline
\hline 
\end{tabular}\vspace{-3mm}
\end{table}

A role inclusion axiom (RIA) of the form $R\sqsubseteq S$ is satisfied
by $\mathcal{I}$ if $R^{\mathcal{I}}\subseteq S^{\mathcal{I}}$.
We denote with $\sqsubseteq_{*}$ the transitive, reflexive closure
of $\sqsubseteq$ over $N_{R}$. If $R\sqsubseteq_{*}S$, we call
$R$ a subrole of $S$ and $S$ a superrole of $R$. A general concept
inclusion (GCI) $C\sqsubseteq D$ is satisfied by $\mathcal{I}$ if
$C^{\mathcal{I}}\subseteq D^{\mathcal{I}}$. A role hierarchy $\mathcal{R}$
is a finite set of RIAs.  
A Tbox $\mathcal{T}$
is a finite set of GCIs. A Tbox $\mathcal{T}$ and its associated role hierarchy $\mathcal{R}$ is satisfied by $\mathcal{I}$ (or consistent) 
if each GCI and RIA is satisfied by $\mathcal{I}$.
Such an interpretation $\mathcal{I}$ is then called a model of $\mathcal{T}$.
A concept description $C$ is said to be satisfiable by $\mathcal{I}$
iff $C^{\mathcal{I}}\neq\emptyset$.
An Abox $\mathcal{A}$ is a finite set of assertions of the form ${a\!:C}$
(concept assertion) with $a^{\mathcal{I}}\in C^{\mathcal{I}}$, and
${(a,b)\!:R}$ (role assertion) with $(a^{\mathcal{I}},b^{\mathcal{I}})\in R^{\mathcal{I}}$.
Due to nominals, a concept assertion ${a\!:C}$ can be transformed into
a concept inclusion $\{a\}\sqsubseteq C$ and a role assertion ${(a,b)\!:R}$
into $\{a\}\sqsubseteq\exists R.\{b\}$. Therefore, concept satisfiability
and Abox consistency can be reduced to Tbox consistency by using
nominals. We use $\{o_{1},\ldots,o_{n}\}$ as an abbreviation for
$\{o_{1}\}\sqcup\cdots\sqcup\{o_{n}\}$ and may write $\{o\}$ as
$o$. Moreover, we do not make the unique name assumption; therefore, two nominals might refer to the same individual.

Nominals carry implicit global numerical restrictions. For example,
if $C\sqsubseteq\{o_{1},o_{2},o_{3}\}$ (or $\{o_{1},o_{2},o_{3}\}\sqsubseteq C$), then $o_{1},o_{2},o_{3}$
impose a numerical restriction that there can be at most (or at least, if  $o_{1},o_{2},o_{3}\in N_{o}$ are declared as pair-wise disjoint)
three instances of $C$. These restrictions are global because they
affect the set of all individuals of $C$ in $\Delta^{\mathcal{I}}$.
These implicit numerical restrictions increase reasoning complexity.  \vspace{-1.5mm}

\section{An Algebraic Tableau Calculus for $\mathcal{SHOI}$}

In this section, we present an algebraic tableau calculus for $\mathcal{SHOI}$
that decides Tbox consistency. Since nominals carry numerical
restrictions, algebraic reasoning is used to ensure their semantics.
The algorithm
takes a $\mathcal{SHOI}$ Tbox $\mathcal{T}$ and its role hierarchy $\mathcal{R}$ as input and
tries to create a complete and clash-free completion graph in order
to check Tbox consistency. The reasoner is divided into two modules:
1) Tableau Module (TM), and 2) Algebraic Module (AM).

Let $G=(V,E,\mathcal{L},\mathcal{B})$ be a completion graph for a $\mathcal{SHOI}$
Tbox $\mathcal{T}$ where $V$ is a set of nodes and $E$ a set of edges.
Each node $x\in V$ is labelled with a set of concepts $\mathcal{L}(x)$, and each edge $\left\langle x,y \right\rangle \in E$ with a set of role names $\mathcal{L}(x,y)$. For each node $x\in V$, if $\mathcal{L}(x)$ contains a universal restriction on role $R$ and there exists an $R$-neighbour of $x$, then $\mathcal{B}(x)$ contains a tuple of the form $\left\langle v,\mathcal{L}(x,v)\right\rangle $ 
where $v \in V$ is an $R$-neighbour of $x$.
We use $\card{v}$ to denote the cardinality of a node $v$. 
For convenience,
we assume that all concept descriptions are in negation normal form. 

TM starts with some preprocessing and reduces all the concept
axioms in a Tbox $\mathcal{T}$ to a single axiom $\top\sqsubseteq C_{\mathcal{T}}$
such that $C_{\mathcal{T}}\coloneqq\bigsqcap_{C\sqsubseteq D\in\mathcal{T}}\mathit{nnf}(\neg C\sqcup D)$, where $\mathit{nnf}$ transforms a given concept expression to its negation normal form.
The algorithm checks consistency of $\mathcal{T}$ by testing
the satisfiability of $o\sqsubseteq C_{\mathcal{T}}$ where $o\in N_{o}$
is a fresh nominal in $\mathcal{T}$, which means that at least $o^{\mathcal{I}} \in {C_{\mathcal{T}}}^{\mathcal{I}}$ and ${C_{\mathcal{T}}}^{\mathcal{I}} \neq \emptyset$. Moreover, since ${\top^{\mathcal{I}}} = {\Delta^{\mathcal{I}}}$ then every domain
element must also satisfy $C_{\mathcal{T}}$.
For creating a complete and
clash-free completion graph, TM applies expansion rules (see Figure \ref{Exp_Rules} and Section
\ref{sub:Expansion-Rules}). AM handles all numerical restrictions using ILP.
It generates inequalities and solves them using the branch-and-price
technique (see Section \ref{sub:Generating-Inequalities} for details). We use equality blocking \cite{horrocks1999description,hladik2004tableau} due to the presence of inverse roles.
\vspace{-4mm}

\subsection{Expansion Rules \label{sub:Expansion-Rules}} 

In order to check the consistency of a Tbox $\mathcal{T}$, the proposed algorithm 
creates a completion graph $G$ using the expansion rules shown in Figure
\ref{Exp_Rules}. A node $x$ in $G$ contains a clash if $\{A,\neg A\} \subseteq \mathcal{L}(x)$ for  $A \in N_{C}$ or AM has no feasible solution for $x$. $G$ is complete if no expansion rule is applicable to any node in $G$. $\mathcal{T}$ is consistent if $G$ is complete and no node in $G$ contains a clash.

The $\sqcap$-Rule, $\sqcup$-Rule and $\forall$-Rule are similar
to standard tableau expansion rules for $\mathcal{ALC}$. The
$\forall_{+}$-Rule preserves the semantics of transitive roles. The
$\mathit{nom_\mathit{merge}}$\textbf{-Rule} merges two nodes containing in their label the same nominal.
Suppose there is $o \in \mathcal{L}(x)$ and $o \in \mathcal{L}(y)$, and nodes $x$ and $y$ are not the same, then $\mathit{nom_\mathit{merge}}$-Rule
merges $x$ into $y$. It adds $\mathcal{L}(x)$ to
$\mathcal{L}(y)$ and moves all edges leading to (from) $x$ so
that they lead to (from) $y$. 
For each node $z$, if $\left\langle z,y \right\rangle \in E$ and $\left\langle z,x \right\rangle \in E$, then $ \mathcal{L}(z,y) =  \mathcal{L}(z,y) \cup  \mathcal{L}(z,x)$. Similarly, if $ \left\langle y,z \right\rangle \in E$ and $ \left\langle x,z \right\rangle \in E$, then $ \mathcal{L}(y,z) =  \mathcal{L}(y,z) \cup  \mathcal{L}(x,z)$. It also merges $\mathcal{B}(x)$ into $\mathcal{B}(y)$.

\begin{figure}[t]
\begin{centering}
\begin{tabular}{|c>{\raggedright}p{0.45\paperwidth}|}
\hline 
$\sqcap$-Rule & \textbf{if} $(C_{1}\sqcap C_{2})\in\mathcal{L}(x)$ and $\{C_{1},C_{2}\}\nsubseteq\mathcal{L}(x)$

\textbf{then} set $\mathcal{L}(x)=\mathcal{L}(x)\cup\{C_{1},C_{2}\}$\tabularnewline
$\sqcup$-Rule & \textbf{if} $(C_{1}\sqcup C_{2})\in\mathcal{L}(x)$ and $\{C_{1},C_{2}\}\cap\mathcal{L}(x)=\emptyset$

\textbf{then} set $\mathcal{L}(x)=\mathcal{L}(x)\cup\{C\}$ for some
$C\in\{C_{1},C_{2}\}$ \tabularnewline
$\forall$-Rule & \textbf{if} $\forall S.C\in\mathcal{L}(x)$ and there 
$\exists y$ with $R\in\mathcal{L}(x,y)$, $C\notin\mathcal{L}(y)$ and
$R\sqsubseteq_{*}S$ \newline
\textbf{then} set $\mathcal{L}(y)=\mathcal{L}(y)\cup\{C\}$\tabularnewline
$\forall_{+}$-Rule & \textbf{if} $\forall S.C\in\mathcal{L}(x)$ and there exist  $U,R$
with $R\in N_{R_{+}}$ and $U\sqsubseteq_{*}R$, $R\sqsubseteq_{*}S$, and 
a node $y$ with $U\in\mathcal{L}(x,y)$ and $\forall R.C\notin\mathcal{L}(y)$\newline
\textbf{then} set $\mathcal{L}(y)=\mathcal{L}(y)\cup\{\forall R.C\}$\tabularnewline
$\mathit{nom_\mathit{merge}}$-Rule & \textbf{if} for some $o\in N_{o}$ there are nodes $x$, $y$ with
$o\in\mathcal{L}(x)\cap\mathcal{L}(y)$, $x\neq y$
\textbf{then} if $x$ is an initial node, then merge $y$ into $x$, else merge $x$ into $y$ \tabularnewline

$\mathit{inverse}$-Rule & \textbf{if} $\forall R^{-}.C\in\mathcal{L}(y)$, $R\in\mathcal{L}(x,y)$, and $\left\langle x,\mathcal{L}(y,x) \right\rangle \notin \mathcal{B}(y)$ \newline\textbf{then} set $\mathcal{B}(y)=\mathcal{B}(y)\cup \{\left\langle x, \mathcal{L}(y,x) \right\rangle \}$ 
\tabularnewline

$\mathit{fil}$-Rule & \textbf{if} $\left\langle \mathsf{R,C},n,\mathsf{V}\right\rangle \in\sigma(x)$ and 
$x$ is not blocked \textbf{then} \vspace{-2.75mm} 
\begin{enumerate}
\item \textbf{if} $\mathsf{V}=\emptyset$ and there exists no $\mathsf{R}$-neighbour $y$ of $x$ with $\mathsf{C}\subseteq\mathcal{L}(y)$, $\card{y}\geq n$, \textbf{then} create a new node $y$ with $\mathcal{L}(y)\leftarrow \mathsf{C}$ and $\card{y}\leftarrow n$
\item \textbf{else} for all $v \in \mathsf{V}$ add $\mathsf{C}$ to $\mathcal{L}(v)$ and set $\card{v}=n$\vspace{-3.2mm} 
\end{enumerate}
\tabularnewline

$e$-Rule & \textbf{if} $\left\langle \mathsf{R,C},n,\mathsf{V}\right\rangle \in\sigma(x)$ and $\mathsf{C}\subseteq\mathcal{L}(y)$,
$\card{y}\geq n$, $\mathsf{R}\nsubseteq\mathcal{L}(x,y)$

\textbf{then} merge $\mathsf{R}$ into $\mathcal{L}(x,y)$ and $\{\mathsf{Inv}(R) \mid R \in \mathsf{R}\}$
into $\mathcal{L}(y,x)$, and for all $S$ with $R\sqsubseteq_{*}S\in\mathcal{R}$ add $S$ to $\mathcal{L}(x,y)$ and $\mathsf{Inv}(S)$ to $\mathcal{L}(y,x)$ \tabularnewline
\hline 
\end{tabular}
\par\end{centering} 
\caption{The expansion rules for $\mathcal{SHOI}$}
\label{Exp_Rules}\vspace{-4mm}
\end{figure}

If $\mathcal{L}(x,y) = \{R\}$ and $\forall R^{-}.C\in\mathcal{L}(y)$, then
the $\mathit{inverse}$\textbf{-Rule} encodes for AM the already existing $R^{-}$-edge 
by adding a tuple $\left\langle x,\{R^{-}\}\right\rangle $ to  $\mathcal{B}(y)$.
AM plays also an important role if nominals occur 
in universal restriction. For example, consider the axioms $A\sqsubseteq\exists R.B$, $B\sqsubseteq\exists R^{-}.C\sqcap\exists R^{-}.D\sqcap\forall R^{-}.\{o_{1}, o_{2}\}$ and $o_{1}\sqcap o_{2}\sqsubseteq \bot$,
where $A,B,C,D\in N_{C}$, $o_{1},o_{2}\in N_{o}$ and $R\in N_{R}$. Suppose we have $A \in \mathcal{L}(x)$,  $R \in \mathcal{L}(x,y)$ and $B \in \mathcal{L}(y)$. Since nominals carry numerical restrictions, $\forall R^{-}.\{o_{1}, o_{2}\}$
implies that we can have at most 2 $R^{-}$-neighbours of $y$.
However, standard tableau reasoners might create two new $R^{-}$-neighbours
of $y$ without considering the existing $R^{-}$-neighbour $x$ of $y$. Then they try to merge these three nodes in a non-deterministic way to satisfy
the numerical restriction imposed by nominals. In our approach, the
$\mathit{inverse}$-Rule encodes information about an existing $R^{-}$-neighbour
of $y$ and AM generates a deterministic solution. 

For a node $x$, AM transforms all existential restrictions, universal restrictions
and nominals to a corresponding system of inequalities. AM then processes
these inequalities and gives back a solution set $\sigma(x)$. The set $\sigma(x)$ is either empty or contains
solutions derived from feasible inequalities. In case of infeasibility AM signals a clash. A solution is defined by a set
of tuples of the form $\left\langle \mathsf{R},\mathsf{C},n,\mathsf{V}\right\rangle $
with $\mathsf{R}\subseteq N_{R}$, 
$\mathsf{C}\subseteq N$, $n\in \mathbb{N}$, $n \geq 1$ and $\mathsf{V}\subseteq V$. Each tuple represents $n$ $\mathsf{R}$-neighbours of $x$ (where $\mathsf{R}$ is a set of roles) that are instances of all elements of $\mathsf{C}$. Here, $\mathsf{V}$ is an optional set that contains existing $\mathsf{R}$-neighbours of $x$ that must be reused and $\mathsf{C}$ is added to their labels. Consider
the axiom $A\sqsubseteq\exists R.B\sqcap\exists R.C\sqcap\forall S.\left\{ o\right\} $,
where $A,B,C\in N_{C}$, $o\in N_{o}$,  $R,S\in N_{R}$, $R \sqsubseteq S$, and $A\in\mathcal{L}(x)$. AM returns
the solution $\sigma(x)=\{\left\langle \{R,S\},\{B,C,o\},1\right\rangle \}$.
The\textbf{ $\mathit{fil}$-Rule} is used to generate nodes based on the arithmetic
solution that satisfies a set of inequalities. For the above solution, 
the $\mathit{fil}$-Rule creates one node $y$ with cardinality 1 such that
$\mathcal{L}(y)\leftarrow\{B,C,o\}$ and $\card{y}=1$. The\textbf{ $e$-Rule}
creates an edge between nodes $x$ and $y$, and adds $R,S$ to $\mathcal{L}(x,y)$
and $\mathsf{Inv}(R),\mathsf{Inv}(S)$ to $\mathcal{L}(y,x)$. The $e$-Rule always adds all implied superroles to edge labels.
\vspace{-3mm}

\subsection{Generating Inequalities \label{sub:Generating-Inequalities}}

Dantzig and Wolfe \cite{dantzig1960decomposition} proposed a column generation technique for solving linear programming (LP) problems, called
Dantzig\textendash Wolfe decomposition, where a large LP is decomposed into
a master problem and a subproblem (or pricing problem). In case of LP problems with a huge number of variables, column generation works with a small subset of variables and builds a Restricted Master Problem (RMP). The Pricing Problem (PP) generates a new variable with the most reduced cost if added to RMP (see \cite{chvatal1983linear,vlasenko2017pushing} for details).
However, column generation may not necessarily give an integral solution for an LP relaxation, i.e., at least one variable has not an integer value.
Therefore, the branch-and-price method \cite{barnhart1998branch}
has been used which is a combination of column generation 
and branch-and-bound technique \cite{desrosiers1984routing}. We employ this technique by mapping number restrictions to linear inequality systems using a column generation ILP formulation (see \cite{vlasenko2017pushing} for details). CPLEX \cite{cplex}
has been used to solve our ILP formulation. \vspace{-3mm}

\subsubsection{Encoding Existential Restrictions and Nominals into Inequalities}

The atomic decomposition technique \cite{Ohlbach-Koehler-99} is used to
encode numerical restrictions on concepts and role fillers into
inequalities. These inequalities are then solved for deciding the
satisfiability of the numerical restrictions. 
The existential restrictions are converted into $\geq1$ inequalities. 
The cardinality of a
partition element containing a nominal $o$ is equal to $1$ due to
the nominal semantics; $\card{\{o\}^{I}}=1$ for each nominal $o\in N_{o}$.
Therefore, the decomposition set is defined as $Q=Q_{\exists}\cup Q_{\forall} \cup Q_{o}$, where
$Q_{\exists}$ ($Q_{\forall}$) contains existential (universal) restrictions and $Q_{o}$ contains
all related nominals. Each element $R_{q}\in Q_{\exists}\cup Q_{\forall}$ represents
a role $R\in N_{R}$ and its qualification concept expression $q$ and each
element $I_{q}\in Q_{o}$ represents a nominal $q\in N_{o}$. The elements in $Q_{\forall}$ are used by AM to ensure the semantics of universal restrictions.
The set of related nominals $Q_{o}\subseteq N_{o}$ is defined as $Q_{o}=\{o\mid o\in clos(q)\wedge R_{q}\in Q_{\exists} \cup Q_{\forall}\}$
where $clos(q)$ is the closure of concept expression $q$. The atomic decomposition
considers all possible ways to decompose $Q$ into sets that are semantically pairwise disjoint. \vspace{-3mm}

\subsubsection{Branch-and-Price Method}

In the following, we use a Tbox $\mathcal{T}$ and its role hierarchy $\mathcal{R}$, a completion graph $G$, a decomposition set $Q$ and a partitioning $\mathcal{P}$
that is the power set of $Q$ containing all subsets of $Q$ except the
empty set. Each partition element $p\in\mathcal{P}$ represents the
intersection of its elements. We decompose our problem into two subproblems:
(i) restricted master problem (RMP), and (ii) pricing problem (PP).
RMP contains a subset of columns and PP computes a column
that can maximally reduce the cost of RMP's objective. Whenever a column
with negative reduced cost is found, it is added to RMP. Number restrictions
are represented in RMP as inequalities, with a restricted set of variables. The flowchart in Figure~\ref{AR} illustrates the whole process. \vspace{-3mm} 

\begin{figure}[t]
\begin{centering}
\includegraphics[scale=0.635]{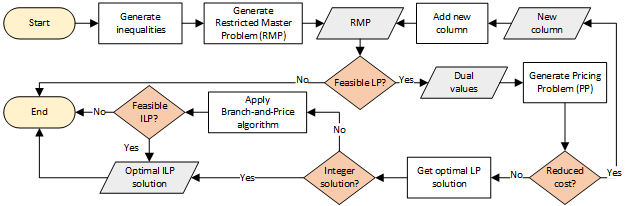} 
\par\end{centering}
\vspace*{-3mm}
\caption{Overview of the algebraic reasoning process}
\label{AR}\vspace{-4mm}
\end{figure}

\subsubsection*{Restricted Master Problem}

RMP is obtained by considering only variables $x_{p}$ with $p \in \mathcal{P}'$ and $\mathcal{P}'\subseteq\mathcal{P}$ and relaxing the
integrality constraints on the $x_{p}$ variables. Initially $\mathcal{P}'$ is empty and RMP contains only artificial variables $h$ to obtain an initial feasible inequality system. Each artificial variable corresponds to
an element in $Q_{\exists} \cup Q_{o}$ such that $h_{R_{q}}$, $R_{q}\in Q_{\exists}$
and $h_{I_{q}}$, $I_{q}\in Q_{o}$. An arbitrarily large cost $M$
is associated with every artificial variable. If any of these artificial
variables exists in the final solution, then the problem is 
infeasible. The objective of RMP is defined as the sum of
all costs as shown in \eqref{eq:rmp_objA-1} of the RMP below.
\begin{equation}
\text{Min} \sum_{p\in\mathcal{P}'} cost_{p}x_{p}+ M\!\!\sum_{R_{q}\in Q_{\exists}} h_{R_{q}}+ M\!\!\sum_{I_{q}\in Q_{o}}h_{I_{q}} \quad\text{subject to} \label{eq:rmp_objA-1}
\end{equation}
\vspace*{-5mm}
\begin{alignat}{2}
\sum_{p\in\mathcal{P}'}a_{p}^{R_{q}}x_{p}+h_{R_{q}} & \geq &\ 1 & \quad  R_{q}\in Q_{\exists} \label{eq:rmpCon1-1}\\
\sum_{p\in\mathcal{P}'}a_{p}^{I_{q}}x_{p}+h_{I_{q}} & = &\ 1 & \quad I_{q}\in Q_{o} \label{eq:rmpCon3-1}
\end{alignat}
\vspace*{-4mm}
\begin{equation}
x_{p}\in\mathbb{R}^{+} \text{ with } p\in\mathcal{P}'\label{eq:lpEx1c-2-1}
\end{equation}
\begin{equation*}
a_{p}^{R_{q}},a_{p}^{I_{q}}\in\{0,1\},\, h_{R_{q}},h_{I_{q}}\in\mathbb{R}^{+} \text{ with } R_{q}\in Q_{\exists},\, I_{q}\in Q_{o}\label{eq:lpEx1c-2-1-1}
\end{equation*}

\noindent where a decision variable $x_{p}$ represents the elements of the partition element $p \in \mathcal{P}'$. The coefficients $a_{p}$ are associated with variables $x_{p}$ and $a_{p}^{R_{q}}$
indicates whether an $R$-neighbour that is an instance of $q$ exists
in $p$. Similarly, $a_{p}^{I_{q}}$ indicates whether
a nominal $q$ exists in $p$. The weight $cost_{p}$ defines the cost
of selecting $p$ and it depends on the number of elements  $p$ contains. Since we minimize the objective function,  $cost_{p}$ in the objective  \eqref{eq:rmp_objA-1} ensures that only
subsets with entailed concepts will be added which are the minimum number of concepts that are needed to satisfy all the axioms. Constraint \eqref{eq:rmpCon1-1}
encodes existential restrictions and \eqref{eq:rmpCon3-1}
 numerical restrictions imposed by nominals (i.e., $\card{\{o\}^{I}}=1$).
Constraint (\ref{eq:lpEx1c-2-1}) states the integrality condition relaxed from $x_{p}\in\mathbb{Z}^{+}$ to $x_{p}\in\mathbb{R}^{+}$.

\subsubsection*{Pricing Problem:}

The objective of PP uses the dual values $\pi, \omega$ as coefficients
of the variables that are associated with a potential partition element.
The binary variables $r_{R_{q}}$, $r_{I_{q}}$, $b_{q}$ ($q\in N$)
are used to ensure the description logic semantics. A binary variable $r_{R_{\top}}$ is used to handle role hierarchy. A variable $b_{q}$ is set to 1 if there exists an instance of concept $q$ and $r_{R{}_{q}}$
is set to 1 if there exists an $R$-neighbour that is an instance of concept
$q$. Likewise, $r_{I_{q}}$ is set to 1 if there exists
a nominal $q$. Otherwise these variable are set to 0. The PP is given below.
\begin{equation}
\text{Min} \sum_{q\in N}b_{q}-\sum_{R_{q}\in Q_{\exists}}\pi_{R_{q}}r_{R_{q}}-\sum_{I_{q}\in Q_{o}}\omega_{I_{q}}r_{I_{q}} \quad\text{subject to} \label{eq:rmp_objA-1-1}
\end{equation}
\vspace*{-5mm}
\begin{alignat}{2}
r_{R_{q}} - b_{q} & \leq &\ 0 & \qquad R_{q}\in Q_{\exists},R\in N_{R},q\in N_{C} \label{eq:existPP}\\
r_{I_{q}}-b_{q} & = &\ 0 & \qquad I_{q}\in Q_{o},q\in N_{o} \label{eq:nominalPP}\\
r_{R_{q}}-r_{R_{\top}} & \leq &\ 0 & \qquad R\in N_{R},q\in N \label{eq:roleT}\\
r_{R_{\top}}-b_{q} & \leq &\ 0 & \qquad R_{q}\in Q_{\forall}, R\in N_{R},q\in N \label{eq:forAll}\\
r_{R_{\top}}-r_{S_{\top}} & \leq & 0 & \qquad R\sqsubseteq S\in\mathcal{R}, R,S\in N_{R}  \label{eq:hierarchy}
\end{alignat}
\vspace*{-6mm}
\begin{equation*}
b_{q},r_{R_{q}},r_{I_{q}},r_{R_{\top}},r_{S_{\top}}\in\left\{ 0,1\right\} 
\end{equation*}

\noindent where vector $\pi$ and $\omega$ are dual variables associated with
 \eqref{eq:rmpCon1-1} and \eqref{eq:rmpCon3-1} respectively.
For each at-least restriction represented in \eqref{eq:rmpCon1-1}, Constraint \eqref{eq:existPP} is added to PP, which ensures
that if $r_{R_{q}}=1$ then variable for $b_{q}$ must exist in $\mathcal{P}'$.
Similarly, \eqref{eq:nominalPP} ensures the semantics
of nominals represented in \eqref{eq:rmpCon1-1}. Constraints
\eqref{eq:roleT} - \eqref{eq:hierarchy} ensure the semantics
of universal restrictions and role hierarchies respectively. 

We can also map the semantics of selected DL axioms, where only atomic concepts occur, into inequalities, as shown in Table \ref{dlToPP}. For every $\mathcal{T} \models A \sqcap B \sqsubseteq C$, AM adds $b_{A} + b_{B}-1 \leq b_{C}$ to PP. Therefore, if PP generates a partition containing $A$ and $B$, then it must also contain $C$. Similarly, for every $\mathcal{T} \models A\sqsubseteq B \sqcup C$, AM adds $b_{A} \leq  b_{B} + b_{C}$ to PP. This inequality ensures that if a partition contains $A$, then it must also contain $B$ or $C$. \vspace{-3mm}

\begin{table}[t]
\caption{DL axioms and their corresponding PP inequalities ($n \ge 1$)}
\label{dlToPP}\vspace{-2mm}

\centering{}%
\begin{tabular}{|>{\centering}m{0.2\columnwidth}|>{\centering}m{0.3\columnwidth}|>{\centering}m{0.45\columnwidth}|}
\hline 
DL Axiom & Inequality in PP & Description\tabularnewline
\hline 
\hline 
\hline 
{\small{}$A_{1}\sqcap...\sqcap A_{n}\sqsubseteq B$} & {\small{}$\sum_{i=1}^{n}b_{A_{i}}-(n-1)\leq b_{B}$} & \raggedright{}{\small{}If a set contains $A_{1},...,A_{n}$,
then it also contains $B$.$^{*}$}\tabularnewline
\hline 
{\small{}$A\sqsubseteq B_{1}\sqcup...\sqcup B_{n}$} & {\small{}$b_{A}\leq\sum_{i=1}^{n}b_{B_{i}}$} & \raggedright{}{\small{}If a set contains $A$,
then it also contains at least one concept from $B_{1},...,B_{n}$}.\tabularnewline
\hline 
\multicolumn{3}{l}{$^{*}${Encodes unsatisfiability and disjointness in case $B\sqsubseteq_{*}\bot$}}
\end{tabular}
\vspace{-4mm}
\end{table}

\subsubsection{Soundness and Completeness of Algebraic Module}

All existential restrictions and nominals are converted into linear
inequalities and added to RMP. Other axioms, such as universal restrictions,
role hierarchy, subsumption and disjointness, are embedded in PP.
In case of feasible inequalities, the branch-and-price algorithm returns
a solution set that contains valid partition elements. Since the branch-and-price
algorithm satisfies all the axioms embedded in RMP and PP, this solution
is sound. Moreover, it is also complete because CPLEX is used to solve
linear inequalities and it does not overlook any possible solution. 
\begin{proposition}
For a set of inequalities, the arithmetic module either generates
an optimal solution which satisfies all inequalities or detects infeasibility. 
\end{proposition} \vspace{-3mm}

\subsection{Example Illustrating Rule Application and ILP formulation  \label{sub:example}}  

Consider the small Tbox \vspace{-2mm}
\begin{gather*}
A\sqsubseteq\exists R.B\sqcap\exists R.\{o1\}\\
B\sqcap\{o1\}\sqsubseteq\bot\\
B\sqsubseteq\exists R.C\sqcap\exists R^{-}.D\sqcap\forall S^{-}.\{o2\}\\
C\sqcap D\sqsubseteq\bot\\
C\sqsubseteq\exists R.E \\
E\sqsubseteq\forall S^{-}.\left\{ o1\right\} 
\end{gather*}
 with $N_{R}=\{R,S\}$, $\{A,B,C,D,E\} \subseteq N_{C},\{o1,o2\} \subseteq N_{o}$,  and $R\sqsubseteq S\in\mathcal{R}$. 
For the sake of better readability, we apply in this example lazy unfolding \cite{baader1994empirical,horrocks1998using}.
 
\begin{enumerate}
\item We start with root node $x$ and its label $\mathcal{L}(x)=\{A\}$ and by
unfolding $A$ and applying the $\sqcap$-Rule we get $\mathcal{L}(x)=\{A,\exists R.B,\exists R.\{o1\}\}$. 

\item Since $\{\exists R.B,\exists R.\{o1\}\} \subseteq \mathcal{L}(x)$, AM generates a corresponding set of inequlities and applies ILP considering known subsumption and disjointness. 

\item For solving these inequalities, RMP starts with
artificial variables, $\mathcal{P}'$ is initially empty, and 
$Q_{\exists}=\{R_{B},R_{o1}\}$, $Q_{\forall}=\emptyset$ and $Q_{o}=\{I_{o1}\}$ (see Fig.\ \ref{fig:ilp-1}).  The objective of (PP 1a) uses the dual values from (RMP 1a). For each at-least
restriction a constraint (e.g., $\exists R.B \leadsto r_{R_{B}}-b_{B}\leq 0$)
is added to (PP 1a), which indicates that if $r_{R_{B}}=1$ then a
variable $b_{B}$ will also be 1. Constraint ($i$) ensures that $B$
and $o1$ cannot exist in same partition element. Constraint ($ii$)
ensures the semantics of nominals.

\begin{figure}[tp]
\begin{minipage}[t][1\totalheight][c]{1\columnwidth}%
\begin{flushleft}
{\small{}}%
\shadowbox{\begin{minipage}[t][1\totalheight][c]{0.46\columnwidth}%
\begin{center}
\textbf{\small{}RMP 1a}{\small{}}\\
{\small{}}%
\begin{minipage}[t]{1\columnwidth}%
\begin{flushleft}
\vspace{-3mm}{\small{}Min \setlength\abovedisplayskip{0pt} \vspace{-2mm}
\[
10h_{R_{B}}+10h_{R_{o1}}+10h_{I_{o1}}\vspace{-2mm}
\]
 }\vspace{-4mm}{\small{}Subject to:}
\par\end{flushleft}%
\end{minipage}{\small{}}\\
{\small{}}%
\begin{minipage}[t][1\totalheight][c]{0.99\columnwidth}%
\begin{center}
{\small{}$\begin{gathered}h_{R_{B}}\geq1\\
h_{R_{o1}}\geq1\\
h_{I_{o1}}=1
\end{gathered}
$}
\par\end{center}%
\end{minipage}{\small{}}\\
{\small{}}%
\fbox{\begin{minipage}[t][1\totalheight][c]{0.99\columnwidth}%
\textbf{\small{}Solution:}{\small{} $cost=30$, $h_{R_{B}}=1$, \newline $h_{R_{o1}}=1,  h_{I_{o1}}=1$}{\small \par}

\textbf{\small{}Duals}{\small{}: $\pi_{R_{B}}=10,\pi_{R_{o1}}=10,\omega_{I_{o1}}=10$}{\small \par}%
\end{minipage}}
\par\end{center}%
\end{minipage}}{\small{}}%
\shadowbox{\begin{minipage}[t][1\totalheight][c]{0.45\columnwidth}%
\begin{center}
\textbf{\small{}PP 1a} \vspace{0.85mm}{\small{}}\\
{\small{}}%
\begin{minipage}[t]{1\columnwidth}%
\begin{flushleft}
\vspace{-3mm}{\small{}Min \setlength\abovedisplayskip{0pt}
\[
b_{B}+b_{o1}-10r_{R_{B}}-10r_{R_{o1}}-10r_{I_{o1}} \vspace{-0.80mm}
\]
 }{\small{}Subject to:}\vspace{0.5mm}
\par\end{flushleft}%
\end{minipage}{\small{}}\\
{\small{}}%
\begin{minipage}[t][1\totalheight][c]{1\columnwidth}%
\begin{center}
{\small{}}%
\begin{tabular}{c|c}
{\small{}$\begin{gathered}r_{R_{B}}-b_{B}\leq 0\\
r_{R_{o1}}-b_{o1}\leq 0
\end{gathered}
$} ~ & ~ {\small{}$\begin{gathered}\begin{aligned}b_{B}+b_{o1} & \leq 1\;~~(i)\\
r_{I_{o1}}-b_{o1} &= 0\; ~~ (ii)
\end{aligned}
\end{gathered}
$}\tabularnewline
\end{tabular}
\par\end{center}%
\end{minipage}{\small{}}\\
\vspace{1mm}{\small{}}%
\fbox{\begin{minipage}[t][1\totalheight][c]{0.99\columnwidth}%
\textbf{\small{}Solution:}{\small{} $cost=-19$, $r_{R_{B}}=0$,\newline $r_{R_{o1}}=1$,
$r_{I_{o1}}=1$, $b_{B}=0$, $b_{o1}=1$}%
\end{minipage}}
\par\end{center}%
\end{minipage}}
\par\end{flushleft}%
\end{minipage}{\small{} }{\small \par}
\caption{Node $x$: First ILP iteration}
\label{fig:ilp-1}  \vspace{-2mm}
\end{figure}

\item The values of $r_{R_{o1}},r_{I_{o1}}$ are 1 in (PP 1a), therefore, the variable $x_{R_{o1}I_{o1}}$ is
added to (RMP 1b). Since only one $b$ variable (i.e.,
$b_{o1}$) is 1, the cost of $x_{R_{o1}I_{o1}}$ is 1. $\mathcal{P}'=\{\{R_{o1},I_{o1}\}\}$
and the value of the objective function is reduced from 30
in (RMP 1a) to 11 in (RMP 1b).

\begin{figure}[tp]
\begin{minipage}[t][1\totalheight][c]{1\columnwidth}%
\begin{flushleft}
{\small{}}%
\shadowbox{\begin{minipage}[t][1\totalheight][c]{0.48\columnwidth}%
\begin{center}
\textbf{\small{}RMP 1b}{\small{}}\\
{\small{}}%
\begin{minipage}[t]{1\columnwidth}%
\begin{flushleft}
\vspace{-5mm}{\small{}Min \setlength\abovedisplayskip{0pt} 
\[
x_{R_{o1}I_{o1}}+10h_{R_{B}}+10h_{R_{o1}}+10h_{I_{o1}}
\]
 }\vspace{-4mm}{\small{}Subject to:}
\par\end{flushleft}%
\end{minipage}{\small{}}
\hspace*{2mm}
{\small{}}%
\begin{minipage}[t][1\totalheight][c]{0.99\columnwidth}%
\begin{center}
{\small{}\setlength\abovedisplayskip{0pt} }%
\begin{tabular}{c}
{\small{}$\;$ $\begin{aligned} & \:h_{R_{B}}\geq1\\
 & x_{R_{o1}I_{o1}}+h_{R_{o1}}\geq1\\
 & x_{R_{o1}I_{o1}}+h_{I_{o1}}=1
\end{aligned}
$ }\tabularnewline
\end{tabular}
\par\end{center}%
\end{minipage}{\small{}}\\
{\small{}}%
\fbox{\begin{minipage}[t][1\totalheight][c]{0.99\columnwidth}%
\textbf{\small{}Solution:}{\small{} $cost=11$, \, $x_{R_{o1}I_{o1}}=1$, \, $h_{R_{B}}=1$,  $h_{R_{o1}}$, $h_{I_{o1}}=0$}{\small \par}

\textbf{\small{}Duals}{\small{}: $\pi_{R_{B}}=10,\, \pi_{R_{o1}}=10,\, \omega_{I_{o1}}=-9$}{\small \par}%
\end{minipage}}
\par\end{center}%
\end{minipage}}{\small{}}%
\shadowbox{\begin{minipage}[t][1\totalheight][c]{0.44\columnwidth}%
\begin{center}
\textbf{\small{}PP 1b} \vspace{0.85mm}{\small{}}\\
{\small{}}%
\begin{minipage}[t]{1\columnwidth}%
\begin{flushleft}
\vspace{-4mm}{\small{}Min \vspace{1mm} \setlength\abovedisplayskip{0pt} 
\[
b_{B}+b_{o1}-10r_{R_{B}}-10r_{R_{o1}}+9r_{I_{o1}} \vspace{-0.30mm}
\]
 }\vspace{-0mm}{\small{}Subject to:} \vspace{1mm}
\par\end{flushleft}%
\end{minipage}{\small{}}\\
{\small{}}%
\begin{minipage}[t][1\totalheight][c]{1\columnwidth}%
\begin{center}
{\small{}\setlength\abovedisplayskip{0pt}}%
\begin{tabular}{c|c}
{\small{}$\begin{gathered}r_{R_{B}}-b_{B}\leq 0\\
r_{R_{o1}}-b_{o1}\leq 0
\end{gathered}
$} ~ & ~ {\small{}$\begin{gathered}\begin{aligned}b_{B}+b_{o1} & \leq 1\;~~(i)\\
r_{I_{o1}}-b_{o1} &= 0\; ~~ (ii)
\end{aligned}
\end{gathered}
$}\tabularnewline
\end{tabular}
\par\end{center}%
\end{minipage}{\small{}}\\
\vspace{1mm}{\small{}}%
\fbox{\begin{minipage}[t][1\totalheight][c]{0.99\columnwidth}%
\textbf{\small{}Solution:}{\small{} $cost=-9$, $r_{R_{B}}=1$,\newline $r_{R_{o1}}=0$,
$r_{I_{o1}}=0$, $b_{B}=1$, $b_{o1}=0$}%
\end{minipage}}
\par\end{center}%
\end{minipage}}
\par\end{flushleft}%
\end{minipage}
\caption{Node $x$: Second ILP iteration}
\label{fig:ilp-2}  \vspace{-4mm}
\end{figure}

\item As the value of $r_{R_{B}}$ is 1 in (PP 1b), the variable $x_{R_{B}}$ is added to
 (RMP 1c). $\mathcal{P}'=\{\{R_{o1}$, $I_{o1}\},\{R_{B}\}\}$
and the cost is further reduced from 11
in (RMP 1b) to 2 in (RMP 1c).

\item All artificial variables in (RMP 1c) are zero which might indicate that we
have reached a feasible solution. The reduced cost of
(PP 1c) is not negative anymore which
means that (RMP 1c) cannot be improved further. Therefore, AM terminates after third ILP iteration and 
returns the optimal solution $\sigma(x)=\{\left\langle \left\{ R\right\}, \left\{ o1\right\} ,1\right\rangle$, $\left\langle \left\{ R\right\} ,\left\{ B\right\} ,1\right\rangle \}$.

\begin{figure}[tp]
\begin{minipage}[t][1\totalheight][c]{1\columnwidth}%
\begin{flushleft}
{\small{}}%
\shadowbox{\begin{minipage}[t][1\totalheight][c]{0.46\columnwidth}%
\begin{center}
\textbf{\small{}RMP 1c}\\
{\small{}}%
\begin{minipage}[t]{1\columnwidth}%
\begin{flushleft}
\vspace{-5mm}{\small{}Min \setlength\abovedisplayskip{0pt}
\[
x_{R_{o1}I_{o1}}+x_{R_{B}}+10h_{R_{B}}+10h_{R_{o1}}+10h_{I_{o1}}
\]
 }\vspace{-0mm}{\small{}Subject to:}
\par\end{flushleft}%
\end{minipage}\\
{\small{}}%
\begin{minipage}[t][1\totalheight][c]{0.99\columnwidth}%
\begin{center}
{\small{}\setlength\abovedisplayskip{0pt}}%
\begin{tabular}{c}
{\small{}$\begin{aligned} & x_{R_{B}}+h_{R_{B}}\geq1\\
 & x_{R_{o1}I_{o1}}+h_{R_{o1}}\geq1\\
 & x_{R_{o1}I_{o1}}+h_{I_{o1}}=1
\end{aligned}
$ }\tabularnewline
\end{tabular}
\par\end{center}%
\end{minipage}\\
{\small{}}%
\fbox{\begin{minipage}[t][1\totalheight][c]{0.99\columnwidth}%
\textbf{\small{}Solution:}{\small{} $cost=2$, $x_{R_{o1}I_{o1}}=1$,
$x_{R_{B}}=1$, $h_{R_{B}},h_{R_{o1}},h_{I_{o1}}=0$ }{\small \par}

\textbf{\small{}Duals}{\small{}: $\pi_{R_{B}}=1,\pi_{R_{o1}}=10,\omega_{I_{o1}}=-9$}{\small \par}%
\end{minipage}}
\par\end{center}%
\end{minipage}}{\small{}}%
\shadowbox{\begin{minipage}[t][1\totalheight][c]{0.45\columnwidth}%
\begin{center}
\textbf{\small{}PP 1c}\\
\vspace{3mm}{\small{}}%
\begin{minipage}[t]{1\columnwidth}%
\begin{flushleft}
\vspace{-5mm}{\small{}Min \vspace{1.2mm} \setlength\abovedisplayskip{0pt}
\[
b_{B}+b_{o1}-1r_{R_{B}}-10r_{R_{o1}}+9r_{I_{o1}}
\]
}\vspace{-0mm}{\small{}Subject to:} \vspace{2mm}
\par\end{flushleft}%
\end{minipage}\\
{\small{}}%
\begin{minipage}[t][1\totalheight][c]{1\columnwidth}%
\begin{center}
{\small{}\setlength\abovedisplayskip{0pt}}%
\begin{tabular}{c|c}
{\small{}$\begin{gathered}r_{R_{B}}-b_{B}\leq 0\\
r_{R_{o1}}-b_{o1}\leq 0
\end{gathered}
$} ~ & ~ {\small{}$\begin{gathered}\begin{aligned}b_{B}+b_{o1} & \leq 1\;~~(i)\\
r_{I_{o1}}-b_{o1} &= 0\; ~~ (ii)
\end{aligned}
\end{gathered}
$}\tabularnewline
\end{tabular}
\par\end{center}%
\end{minipage}\\
\vspace{1.5mm}{\small{}}%
\fbox{\begin{minipage}[t][1\totalheight][c]{0.99\columnwidth}%
\textbf{\small{}Solution:}{\small{} } \vspace{1mm}{\small \par}

{\small{}$cost=0$, all variables are 0.}{\small \par}%
\end{minipage}}
\par\end{center}%
\end{minipage}}
\par\end{flushleft}%
\end{minipage}
\caption{Node $x$: Third ILP iteration}
\label{fig:ilp-3}  \vspace{-4mm}
\end{figure}

\item The $\mathit{fil}$-Rule creates two new nodes $x_{1}$ and $x_{2}$ with $\mathcal{L}(x_{1})\leftarrow\{o1\}$,
$\mathcal{L}(x_{2})\leftarrow\{B\}$, $\card{x_{1}}\leftarrow1$ and $\card{x_{2}}\leftarrow1$. 

\item The $e$-Rule creates edges $\left\langle x,x_{1}\right\rangle $
and $\left\langle x,x_{2}\right\rangle $ with $\mathcal{L}\left(\left\langle x,x_{1}\right\rangle \right)\leftarrow\{R,S\}$
and $\mathcal{L}\left(\left\langle x,x_{2}\right\rangle \right)\leftarrow\{R,S\}$
(because $R\sqsubseteq S\in\mathcal{R}$). It also creates back edges
$\left\langle x_{1},x\right\rangle $ and $\left\langle x_{2},x\right\rangle $
with $\mathcal{L}\left(\left\langle x_{1},x\right\rangle \right)\leftarrow\{R^{-},S^{-}\}$
and $\mathcal{L}\left(\left\langle x_{2},x\right\rangle \right)\leftarrow\{R^{-},S^{-}\}$.

\item By unfolding $B$ in the label of $x_{2}$ and by applying the $\sqcap$-Rule
we get $\mathcal{L}(x_{2})=\{B,\exists R.C,\exists R^{-}.D,\forall S^{-}.\{o2\}\}$. 

\item The $\mathit{inverse}$-Rule encodes information about existing $R^{-}$-neighbour $x$ 
of $x_{2}$ by adding a tuple $\left\langle x,\{R^{-},S^{-}\}\right\rangle $ to  $\mathcal{B}(x_{2})$.

\item AM uses $\{\exists R.C,\exists R^{-}.D\}$ to start ILP. Due to lack of space we cannot provide the complete RMP and
PP solution process here. Since $R\sqsubseteq S$, the universal
restriction $\forall S^{-}.\left\{ o2\right\} $ is ensured by adding
the following inequalities to PP: $r_{R_{\top}^{-}}-r_{S_{\top}^{-}}\leq 0$,  $r_{S_{\top}^{-}}-b_{o2}\leq 0$, 
and for all $r_{R_{q}^{-}}$ we added an equality $r_{R_{q}^{-}}-r_{R_{\top}^{-}}\leq0$.
Therefore, whenever $r_{R_{q}^{-}}=1$ the values of $r_{R_{\top}^{-}},r_{S_{\top}^{-}},b_{o2}=1$ . 

\item Since $\mathcal{B}(x_{2})$ contains $\left\langle x,\{R^{-},S^{-}\}\right\rangle $, AM adds node $x$ in solution. Therefore, AM returns the solution $\sigma(x_{2})=\{\left\langle \left\{ R\right\} ,\left\{ C\right\} ,1\right\rangle ,\left\langle \left\{ R^{-},S^{-}\right\} ,\left\{D,o2\right\} ,1 ,\left\{x\right\} \right\rangle \}$. 

\item The $\mathit{fil}$-Rule creates only one new node $x_{3}$ with $\mathcal{L}(x_{3})\leftarrow\{C\}$
and $\card{x_{3}}\leftarrow1$, and updates the label of node $x$ with $\mathcal{L}(x)\leftarrow\{D,o2\}$. 

\item The $e$-Rule creates edges $\left\langle x_{2},x_{3}\right\rangle $
and $\left\langle x_{3},x_{2}\right\rangle $ with $\mathcal{L}\left(\left\langle x_{2},x_{3}\right\rangle \right)\leftarrow\{R,S\}$
and $\mathcal{L}\left(\left\langle x_{3},x_{2}\right\rangle \right)\leftarrow\{R^{-},S^{-}\}$. 

\item By unfolding $C$ in the label of $x_{3}$ we get $\mathcal{L}(x_{3})=\{C,\exists R.E\}$.
AM gives solution $\sigma(x_{3})=\{\left\langle \left\{ R\right\} ,\left\{ E\right\} ,1\right\rangle \}$.
The $\mathit{fil}$-Rule creates node $x_{4}$ with $\mathcal{L}(x_{4})\leftarrow\{E\}$
and $\card{x_{4}}\leftarrow1$. The $e$-Rule creates edges $\left\langle x_{3},x_{4}\right\rangle $
and $\left\langle x_{4},x_{3}\right\rangle $ with $\mathcal{L}\left(\left\langle x_{3},x_{4}\right\rangle \right)\leftarrow \{R,S\}$
and $\mathcal{L}\left(\left\langle x_{4},x_{3}\right\rangle \right)\leftarrow\{R^{-},S^{-}\}$. 

\item $\mathcal{L}(x_{4})=\{E,\forall S^{-}.\left\{ o1\right\} \}$ and after unfolding
$E$ the $\forall$-Rule adds $o1$ to $\mathcal{L}(x_{3})$.
However, $o1$ already occurs in $\mathcal{L}(x_{1})$ and $x_{1}\neq x_{3}$.
Therefore, the $\mathit{nom_\mathit{merge}}$-Rule merges node $x_{3}$ into node
$x_{1}$. 
\item Since no more rules are applicable, the tableau algorithm terminates. 
\end{enumerate} \vspace{-5mm}

\section{Performance Evaluation }

We developed a prototype system called Cicada\footnote{System and test ontologies: \url{https://users.encs.concordia.ca/~haarslev/Cicada}} that implements our
calculus as proof of concept. Besides the use of ILP and branch-and-price Cicada only implements a few standard optimization techniques such as lazy unfolding \cite{baader1994empirical,horrocks1998using},
nominal absorption \cite{horrocks1998using}, and dependency directed
backtracking \cite{baader2003description} as well as a \textit{ToDo list} architecture \cite{tsarkov2006fact++} to control the application of the expansion rules. Cicada might not perform well for $\mathcal{SHOI}$ ontologies that require other optimization
techniques. 

Therefore, we built a set of synthetic test cases to empirically
evaluate Cicada. Figure \ref{evaluation} presents some metrics
of benchmark ontologies and evaluation results. We compared Cicada
with major OWL reasoners such as FaCT++ (1.6.5) \cite{tsarkov2006fact++},
HermiT (1.3.8) \cite{shearer2008hermit}, and Konclude (0.6.2) \cite{steigmiller2014konclude}.

The first benchmark (see top part of Figure \ref{evaluation}) uses two real-world ontologies. The ontology EU-Members (adapted from \cite{faddoul2010algebraic}) models 28 members of European Union (EU) whereas
CA-Provinces models 10 provinces of Canada. We added nominals requiring 
29 EU members and 11 Canadian provinces respectively. 
The results show that only Cicada can identify the inconsistency of EU-Members  within the time limit. Moreover, Cicada is more than two orders of magnitude faster than FaCT++ in identifying the inconsistency of CA-Provinces.

The second benchmark (see bottom part of Figure \ref{evaluation})
consists of small synthetic test ontologies that are using a variable $n$ for representing the number of nominals. 
In order to test the effect of increased number of nominals we defined concept $C$ and $A$ as
$C\sqsubseteq\exists R^{-}.A$ and $A\sqsubseteq\exists R.X_{1}\sqcap,...,\sqcap\exists R.X_{n}\sqcap\forall R.\{o_{1},...,o_{n}\}$. 
Nominals $o_{1},...,o_{n}$ and concepts $X_{1},...,X_{n}$ are declared as pairwise
disjoint. The first set consists of
 consistent ontologies in which we declared $C$ and $X_{1},...,X_{n-1}$ as pairwise
disjoint. The second set consists of inconsistent
ontologies in which we declared $C$ and $X_{1},...,X_{n}$ as pairwise disjoint.
Only Cicada can process the
ontologies with more than 10 nominals within the time limit. 

\begin{figure}[t]
\begin{centering}
\begin{tabular}{|l|c|c|c|c|c|c|c|}
\hline 
\multirow{2}{*}{\textbf{Ontology Name}} & \multicolumn{3}{c|}{\textbf{Ontology Metrics}} & \multicolumn{4}{c|}{\textbf{Evaluation Results }}\tabularnewline
\cline{2-8} 
 & \textbf{\#Axioms } & \textbf{\#Concepts } & \textbf{\#Ind } & \textbf{~Cic~} & \textbf{~FaC~} & \textbf{~Her~} & \textbf{~Kon~}\tabularnewline
\hline 
EU-Members & {\small{}67} & {\small{}32} & {\small{}28} & {\small{}4.86} & {\small{}TO} & {\small{}TO} & {\small{}TO}\tabularnewline
\cline{2-8} 
CA-Provinces & {\small{}32} & {\small{}14} & {\small{}11} & {\small{}2.85} & {\small{}316.4} & {\small{}TO} & {\small{}TO}\tabularnewline
\hline 
\end{tabular}\\
\begin{tabular}{|c|c|c|c|c|c|c|c|c|c|c|c|}
\hline 
 & \multicolumn{3}{c|}{} & \multicolumn{4}{c|}{\textbf{TestOnt-Cons}} & \multicolumn{4}{c|}{\textbf{TestOnt-InCons}}\tabularnewline
\cline{5-12} 
\multirow{2}{*}{\textbf{n}} & \multicolumn{3}{c|}{\textbf{Ontology Metrics}} & \multicolumn{4}{c|}{\textbf{\small{}Evaluation Results}} & \multicolumn{4}{c|}{\textbf{\small{}Evaluation Results}}\tabularnewline
\cline{2-12} 
 & \textbf{\#Axioms} & \textbf{\#Concepts} & \textbf{\#Ind} & \textbf{Cic } & \textbf{FaC } & \textbf{Her } & \textbf{Kon} & \textbf{Cic } & \textbf{FaC } & \textbf{Her } & \textbf{Kon}\tabularnewline
\hline 
{\small{}40} & {\small{}92} & {\small{}43} & {\small{}41} & {\small{}3.39} & {\small{}TO} & {\small{}TO} & {\small{}TO} & {\small{}4.41} & {\small{}TO} & {\small{}TO} & {\small{}TO}\tabularnewline
{\small{}20} & {\small{}53} & {\small{}23} & {\small{}21} & {\small{}1.21} & {\small{}TO} & {\small{}TO} & {\small{}TO} & {\small{}3.16} & {\small{}TO} & {\small{}TO} & {\small{}TO}\tabularnewline
{\small{}10} & {\small{}33} & {\small{}13} & {\small{}11} & {\small{}0.91} & {\small{}TO} & {\small{}TO} & {\small{}TO} & {\small{}2.68} & {\small{}401.7} & {\small{}TO} & {\small{}TO}\tabularnewline
{\small{}7} & {\small{}27} & {\small{}10} & {\small{}8} & {\small{}0.64} & {\small{}1.26} & {\small{}3.47} & {\small{}3.56} & {\small{}2.32} & {\small{}1.48} & {\small{}3.70} & {\small{}3.71}\tabularnewline
{\small{}5} & {\small{}23} & {\small{}8} & {\small{}6} & {\small{}0.41} & {\small{}0.02} & {\small{}0.13} & {\small{}0.24} & {\small{}2.21} & {\small{}0.12} & {\small{}0.46} & {\small{}0.14}\tabularnewline
\hline 
\end{tabular}
\par\end{centering}
\vspace{-2mm}

\caption{Metrics of Benchmark Ontologies and Evaluation Results with runtime
in seconds and a timeout of 1000 seconds (TO=timeout, \#=Number of...,
Ind=Individuals, Cic=Cicada, FaC=FaCT++, Her=HermiT, Kon=Konclude)}
\label{evaluation} \vspace{-3mm}
\end{figure}

\section{Conclusion} 

We presented a tableau-based algebraic calculus for handling the numerical
restrictions imposed by nominals, existential and universal restrictions, and their interaction with inverse roles. These numerical restrictions are translated into linear
inequalities which are then solved by using algebraic reasoning. The
algebraic reasoning is based on a branch-and-price technique that
either computes an optimal solution, or detects infeasibility. An
empirical evaluation of our calculus showed that it performs better
on ontologies having a large number of nominals, whereas other reasoners
were unable to classify them within a reasonable amount of time. In future work,
we will extend the technique presented here to $\mathcal{SHOIQ}$. 

\newpage
\bibliographystyle{splncs04}
\bibliography{bibexport}
\newpage
\appendix

\section{Integer Linear Programming}
\textit{Linear Programming} (LP) is the study of determining the
minimum (or maximum) value of a linear function $f(x_{1},x_{2},...,x_{n})$
subject to a finite number of linear constraints. These constraints
consist of linear inequalities involving variables $x_{1},x_{2},...,x_{n}$
\cite{chvatal1983linear}.
If all of the variables are required to have integer values, then
the problem is called \textit{Integer Programming} (IP) or \textit{Integer
Linear Programming} (ILP).

\textit{Simplex method}, proposed by G. B. Dantzig \cite{dantzig1955generalized},
is one of the most frequently used methods to solve LP problems. Although
LP is known to be solvable in polynomial time \cite{khachiyan1980polynomial},
the simplex method can behave exponentially for certain problems.
Karmarkar's algorithm \cite{karmarkar1984new} is the first efficient
polynomial time method for solving a linear program. However, all
these approaches are not reasonably efficient in solving problems
with a huge number of variables. Therefore, the column generation technique is used to solve problems with a huge
number of variables. It decomposes
a large LP into the master problem and the subproblem. It only considers a small subset of variables. 

\subsection{Branch-and-price Method}

The column generation technique may not necessarily give an integral
solution for an LP relaxation. Therefore, we use
branch-and-price method \cite{barnhart1998branch}  
that is hybrid of column generation and branch-and-bound method \cite{desrosiers1984routing}. We decompose our problem into two subproblems:
(i) restricted master problem (RMP), and (ii) pricing problem (PP).
Number restrictions
are represented in RMP as inequalities, with a restricted set of variables. We employ this technique by mapping number restrictions to linear inequality systems using a column generation ILP formulation. 
These inequalities are then solved for deciding the
satisfiability of the numerical restrictions. \vspace{-3mm}

\subsubsection{Column Generation ILP Formulation:}

In the following, we use a Tbox $\mathcal{T}$ and its $\mathcal{R}$, a completion graph $G$, a decomposition set $Q$ and a partitioning $\mathcal{P}$
that is the power set of $Q$ containing all subsets of $Q$ except the
empty set. Each partition element $p\in\mathcal{P}$ represents the
intersection of its elements.
The ILP model associated with the feasibility problem
of $Q$ is as follows: \vspace{2mm}

\hspace{3mm}\text{Min}
\begin{equation}
 \sum_{p\in\mathcal{P}'} cost_{p}x_{p}  \label{eq:rmp_objA-1b}
\end{equation}
\vspace*{-5mm}
\hspace{3mm}\text{Subject to}
\begin{alignat}{2}
\sum_{p\in\mathcal{P}'}a_{p}^{R_{q}}x_{p} & \geq &\ 1 & \quad  R_{q}\in Q_{\exists} \label{eq:rmpCon1-1b}\\
\sum_{p\in\mathcal{P}'}a_{p}^{I_{q}}x_{p} & = &\ 1 & \quad I_{q}\in Q_{o} \label{eq:rmpCon3-1b}
\end{alignat}
\vspace*{-4mm}
\begin{equation}
x_{p}\in\mathbb{R}^{+} \text{ with } p\in\mathcal{P}'\label{eq:lpEx1c-2-1b}
\end{equation}
\begin{equation*}
a_{p}^{R_{q}},a_{p}^{I_{q}}\in\{0,1\},\, \text{ with } R_{q}\in Q_{\exists},\, I_{q}\in Q_{o}\label{eq:lpEx1c-2-1-1b}
\end{equation*}

\noindent here, $Q=Q_{\exists}\cup Q_{\forall} \cup Q_{o}$, where
$Q_{\exists}$ ($Q_{\forall}$) contains existential (universal) restrictions and $Q_{o}$ contains
all related nominals (see Section \ref{sub:Generating-Inequalities} for details). A decision variable $x_{p}$ represents the elements of the partition
element $p\in\mathcal{P}$. $a_{p}$ is associated with each variable
$x_{p}$ and $a_{p}^{R_{q}}$ indicates whether an $R$-successor that
is an instance of $q$ exists in the subset $p$. The weight $cost_{p}$ is the
cost of selecting $p$ and is defined as the number of concepts present
in $p$. Since we minimize, $cost_{p}$ ensures the partition element only contains the minimum
number of concepts that are needed to satisfy all the axioms.
Constraint \eqref{eq:rmpCon1-1b}
encodes existential restrictions and \eqref{eq:rmpCon3-1b}
 numerical restrictions imposed by nominals (i.e., $\card{\{o\}^{I}}=1$).
Constraint (\ref{eq:lpEx1c-2-1b}) states the integrality condition relaxed from $x_{p}\in\mathbb{Z}^{+}$ to $x_{p}\in\mathbb{R}^{+}$.

In order to begin the solution process, we need to find an initial set
of columns satisfying the Constraints \eqref{eq:rmpCon1-1b} and \eqref{eq:rmpCon3-1b}. However, it can be a cumbersome task to find an initial feasible solution. Therefore, we initially start our RMP with artificial variables $h$ (as shown in Section \ref{sub:Generating-Inequalities}) to obtain an initial artificial feasible inequality system. Each artificial variable corresponds to
an element in $Q_{\exists} \cup Q_{o}$ such that $h_{R_{q}}$, $R_{q}\in Q_{\exists}$
and $h_{I_{q}}$, $I_{q}\in Q_{o}$. An arbitrarily large cost $M$
is associated with every artificial variable. The objective of RMP is defined as the sum of
all costs i.e., $ \sum_{p\in\mathcal{P}'} cost_{p}x_{p}+ M\!\!\sum_{R_{q}\in Q_{\exists}} h_{R_{q}}+ M\!\!\sum_{I_{q}\in Q_{o}}h_{I_{q}} $. Since we minimize the objective function, by considering this large cost $M$ one can ensure that as the column generation method
proceeds, the artificial variables will leave the basis. Therefore, in case of feasible set of inequalities, these artificial
variables must not exist in the final solution.
The objective of PP uses the dual values $\pi, \omega$ as coefficients
of the variables that are associated with a potential partition element (shown in Section \ref{sub:Generating-Inequalities}). PP encodes the semantics of universal restrictions, subsumption, disjointness and role hierarchies into inequalities by using binary variables. PP computes a column that can maximally reduce the cost of RMP's objective. Whenever a column with negative reduced cost is found, it is added to RMP. 
The process terminates when PP cannot compute a new column with reduced cost, i.e., when the value of the objective function of PP becomes greater or equal to 0.

\section{Full Reasoning Example}
\label{sec:full-example}
In this Appendix, we provided the detailed explanation of the example presented in Section \ref{sub:example}. Since detailed expansion of the node $x$ has already been provided in Section \ref{sub:example}, we started here by expanding the node  $x_{2}$.
\begin{enumerate}
\item By unfolding $B$ in the label of $x_{2}$ and by applying the $\sqcap$-Rule
we get $\mathcal{L}(x_{2})=\{B,\exists R.C,\exists R^{-}.D,\forall S^{-}.\{o2\}\}$.
\item The $\mathit{inverse}$-Rule adds information about the existing $R^{-}$-neighbour
$x$ of $x_{2}$ by adding a tuple $\left\langle x,\{R^{-},S^{-}\}\right\rangle $ to  $\mathcal{B}(x_{2})$.
\item AM uses $\{\exists R.C,\exists R^{-}.D, \exists R^{-}.X\}$
to start ILP. Here, $\exists R^{-}.X$ is used to handle existing $R^{-}$-edge.
\item RMP starts with artificial variables, $\mathcal{P}'$ is initially empty, and $Q_{\exists}=\{R_{C},R_{D}^{-}$, $R_{X}^{-}\}$,
$Q_{\forall}=\{S_{o2}\}$ and $Q_{o}=\{I_{o2}\}$ (see Fig.\ \ref{fig:ilp-2-a}).
The objective of the (PP 2a) uses the dual values from (RMP 2a). In
(PP 2a), Constraint ($i$) ensures that $C$ and $D$ cannot exist
in same partition element.  Constraint ($ii$) - ($v$) ensures the
semantics of role hierarchy and universal restrictions.
\item The values of $r_{R_{C}}$, $r_{R_{X}^{-}}$, $r_{I_{o2}}$,
$r_{R_{\top}^{-}}$, $r_{S_{\top}^{-}}$ are 1 in (PP 2a), therefore, the variable
$x_{R_{C}R_{X}^{-}I_{o2}}$ is added to (RMP 2b). Since three $b$
variable (i.e., {\small{}$b_{C},b_{X},b_{o2}$}) are 1, the cost of
$x_{R_{C}R_{X}^{-}I_{o2}}$ is 3. The value of the objective function
is reduced from 40 in (RMP 2a) to 13 in (RMP 2b).

\begin{figure}[ht]
\begin{minipage}[t][1\totalheight][c]{1\columnwidth}%
\begin{flushleft}
{\small{}{}}%
\shadowbox{\begin{minipage}[t][1\totalheight][c]{0.95\columnwidth}%
\begin{center}
\textbf{\small{}{}RMP 2a}{\small{}{}}\\
 {\small{}{}}%
\begin{minipage}[t]{1\columnwidth}%
\begin{flushleft}
{\small{}{}Min \setlength{\abovedisplayskip}{0pt} ~~~~~~~~~~~~~~~~~~~
$10h_{R_{C}}+10h_{R_{D}^{-}}+10h_{R_{X}^{-}}+10h_{I_{o2}}$
}\newline \vspace{-4mm}
{\small{}{}Subject to:} 
\par\end{flushleft}%
\end{minipage}{\small{}{}}\\
 {\small{}{}}%
\begin{minipage}[t][1\totalheight][c]{0.99\columnwidth}%
\begin{center}
{\small{}{}$\begin{gathered}h_{R_{C}}\geq1\\
h_{R_{D}^{-}}\geq1\\
h_{R_{X}^{-}}\geq1\\
h_{I_{o2}}=1
\end{gathered}
$} 
\par\end{center}%
\end{minipage}{\small{}{}}\\
 {\small{}{}}%
\fbox{%
\begin{minipage}[t][1\totalheight][c]{0.99\columnwidth}%
\textbf{\small{}{}Solution:}{\small{}{} $cost=40$, $h_{R_{C}}=1$,
$h_{R_{D}^{-}}=1,h_{R_{X}^{-}}=1,h_{I_{o1}}=1$}{\small \par}

\textbf{\small{}{}Duals}{\small{}{}: $\pi_{R_{C}}=10,\pi_{R_{D}^{-}}=10,\pi_{R_{X}^{-}}=10,\omega_{I_{o1}}=10$}{\small \par}%
\end{minipage}} 
\par\end{center}%
\end{minipage}}{\small{}{}}%
\newline
\shadowbox{\begin{minipage}[t][1\totalheight][c]{0.95\columnwidth}%
\begin{center}
\textbf{\small{}{}PP 2a}{\small{}{}}\\
 {\small{}{}}%
\begin{minipage}[t]{1\columnwidth}%
\begin{flushleft}

{\small{}{}Min \setlength{\abovedisplayskip}{0pt} ~~~~~~~~~~~~~~~~~
$b_{C}+b_{D}+b_{X}+b_{o2}-10r_{R_{C}}-10r_{R_{D}^{-}}- 10r_{R_{X}^{-}}-10r_{I_{o2}}$
}\newline
{\small{}{}Subject to:} 
\par\end{flushleft}%
\end{minipage}{\small{}{}}\\
 {\small{}{}}%
\begin{minipage}[t][1\totalheight][c]{1\columnwidth}%
\begin{center}
{\small{}{}}%
\begin{tabular}{c|cc}
{\small{}{}$\begin{gathered}r_{R_{C}}-b_{C}\leq0\\
r_{R_{D}^{-}}-b_{D}\leq0\\
r_{R_{X}^{-}}-b_{X}\leq0\\
r_{I_{o2}}-b_{o2}=0
\end{gathered} 
$} ~~~ & ~~~ {\small{}{}$\begin{gathered}\begin{aligned}b_{C}+b_{D} & \leq1\;~~(i)\\
r_{R_{D}^{-}}-r_{R_{\top}^{-}} & \leq0\;~~(ii)\\
r_{R_{X}^{-}}-r_{R_{\top}^{-}} & \leq0\;~~(iii)\\
r_{R_{\top}^{-}}-r_{S_{\top}^{-}} & \leq0\;~~(iv)\\
r_{S_{\top}^{-}}-b_{o2} & \leq0\;~~(v)
\end{aligned}
\end{gathered}
$} & \, ~~~~~~~~~~~~~~~ (CPP 2a) \tabularnewline
\end{tabular}
\par\end{center}%
\end{minipage}{\small{}{}}\\
 {\small{}{}}%
\fbox{%
\begin{minipage}[t][1\totalheight][c]{0.99\columnwidth}%
\textbf{\small{}{}Solution:}{\small{}{} $cost=-27$, $b_{C}=1$,
$b_{D}=0$, $b_{X}=1$, $b_{o2}=1$, $r_{R_{C}}=1$, $r_{R_{D}^{-}}=0$,
$r_{R_{X}^{-}}=1$, $r_{I_{o2}}=1$, $r_{R_{\top}^{-}}=1$, $r_{S_{\top}^{-}}=1$}%
\end{minipage}} 
\par\end{center}%
\end{minipage}} 
\par\end{flushleft}%
\end{minipage}{\small{}{} }{\small \par}

\caption{Node $x_{2}$: First ILP iteration}
\label{fig:ilp-2-a} 
\end{figure}

\item As the value of $r_{R_{C}}$ is 1, the variable $x_{R_{C}}$ is added
to (RMP 2c). 
\item In third ILP iteration, the values of $r_{R_{D}^{-}}$, $r_{R_{X}^{-}}$,
$r_{I_{o2}}$, $r_{R_{\top}^{-}}$, $r_{S_{\top}^{-}}$ are 1. The
variable $x_{R_{D}^{-}R_{X}^{-}I_{o2}}$ is added to (RMP 2d). Since
three $b$ variables (i.e., $b_{D},b_{X},b_{o2}$) are 1,
the cost of $x_{R_{D}^{-}R_{X}^{-}I_{o2}}$ is 3. The value of the
objective function is reduced from 13 in (RMP 2c) to 4 in (RMP 2d).
\item The values of $r_{R_{D}^{-}}$, $r_{I_{o2}}$, $r_{R_{\top}^{-}}$,$r_{S_{\top}^{-}}$ are 1, therefore, the variable $x_{R_{D}^{-}I_{o2}}$
is added to (RMP 2e). The cost is not reduced further in (RMP 2e). 
\item All artificial variables in (RMP 2e) are zero and the reduced cost
of (PP 2e) is not negative anymore which indicates that the RMP cannot
be improved further. Therefore, AM terminates after the fifth ILP iteration. Since $\mathcal{B}(x_{2})$ contains $\left\langle x,\{R^{-},S^{-}\}\right\rangle $, AM adds node $x$ in solution. Therefore, AM returns the solution $\sigma(x_{2})=\{\left\langle \left\{ R\right\} ,\left\{ C\right\} ,1\right\rangle ,\left\langle \left\{ R^{-}, S^{-}\right\} ,\left\{ D,o2\right\} ,1, \left\{ x\right\}\right\rangle \}$.
\item The $\mathit{fil}$-Rule creates only one new node $x_{3}$ with $\mathcal{L}(x_{3})\leftarrow\{C\}$
and $\card{x_{3}}\leftarrow1$, and updates the label of node $x$ with $\mathcal{L}(x)\leftarrow\{D,o2\}$.
\item The $e$-Rule creates edges $\left\langle x_{2},x_{3}\right\rangle $
and $\left\langle x_{3},x_{2}\right\rangle $ with $\mathcal{L}\left(\left\langle x_{2},x_{3}\right\rangle \right)\leftarrow\{R,S\}$
and $\mathcal{L}\left(\left\langle x_{3},x_{2}\right\rangle \right)\leftarrow\{R^{-},S^{-}\}$.

\begin{figure}[tp]
\begin{minipage}[t][1\totalheight][c]{1\columnwidth}%
\begin{flushleft}
{\small{}{}}%
\shadowbox{\begin{minipage}[t][1\totalheight][c]{0.95\columnwidth}%
\begin{center}
\textbf{\small{}{}RMP 2b}{\small{}{}}\\
 {\small{}{}}%
\begin{minipage}[t]{1\columnwidth}%
\begin{flushleft}
\vspace{-3mm}
{\small{}{}Min \setlength{\abovedisplayskip}{0pt} 
\[
3x_{R_{C}R_{X}^{-}I_{o2}}+10h_{R_{C}}+10h_{R_{D}^{-}}+10h_{R_{X}^{-}}+10h_{I_{o2}}
\]
}\vspace{-4mm}
{\small{}{}Subject to:} 
\par\end{flushleft}%
\end{minipage}{\small{}{}}\\
 {\small{}{}}%
\begin{minipage}[t][1\totalheight][c]{0.99\columnwidth}%
\begin{center}
{\small{}{}$\begin{gathered}x_{R_{C}R_{X}^{-}I_{o2}}+h_{R_{C}}\geq1\\
h_{R_{D}^{-}}\geq1\\
x_{R_{C}R_{X}^{-}I_{o2}}+h_{R_{X}^{-}}\geq1\\
x_{R_{C}R_{X}^{-}I_{o2}}+h_{I_{o2}}=1
\end{gathered}
$} 
\par\end{center}%
\end{minipage}{\small{}{}}\\
 {\small{}{}}%
\fbox{%
\begin{minipage}[t][1\totalheight][c]{0.99\columnwidth}%
\textbf{\small{}{}Solution:}{\small{}{} $cost=13$, $x_{R_{C}R_{X}^{-}I_{o2}}=1,h_{R_{C}}=0$,
$h_{R_{D}^{-}}=1,h_{R_{X}^{-}}=0,h_{I_{o1}}=0$}{\small \par}

\textbf{\small{}{}Duals}{\small{}{}: $\pi_{R_{C}}=0,\pi_{R_{D}^{-}}=10,\pi_{R_{X}^{-}}=10,\omega_{I_{o1}}=-17$}{\small \par}%
\end{minipage}} 
\par\end{center}%
\end{minipage}}{\small{}{}}%
\newline
\shadowbox{\begin{minipage}[t][1\totalheight][c]{0.95\columnwidth}%
\begin{center}
\textbf{\small{}{}PP 2b}{\small{}{}}\\
 {\small{}{}}%
\begin{minipage}[t]{1\columnwidth}%
\begin{flushleft}
\vspace{-2mm}
{\small{}{}Min \setlength{\abovedisplayskip}{0pt} 
\[
b_{C}+b_{D}+b_{X}+b_{o2}-0r_{R_{C}}-10r_{R_{D}^{-}}-10r_{R_{X}^{-}}+17r_{I_{o2}}
\]
}\vspace{-0mm}
{\small{}{}Subject to: ~~~ (CPP 2a)} 
\par\end{flushleft}%
\end{minipage}{\small{}{}}\\
 {\small{}{}}%
\fbox{%
\begin{minipage}[t][1\totalheight][c]{0.99\columnwidth}%
\textbf{\small{}{}Solution:}{\small{}{} $cost=-9$, $b_{C}=1$,
$r_{R_{C}}=1$, all other variables are 0.}%
\end{minipage}} 
\par\end{center}%
\end{minipage}} 
\par\end{flushleft}%
\end{minipage}{\small{}{} }{\small \par}

\caption{Node $x_{2}$: Second ILP iteration}
\label{fig:ilp-2-b} 
\end{figure}

\begin{figure}[tp]
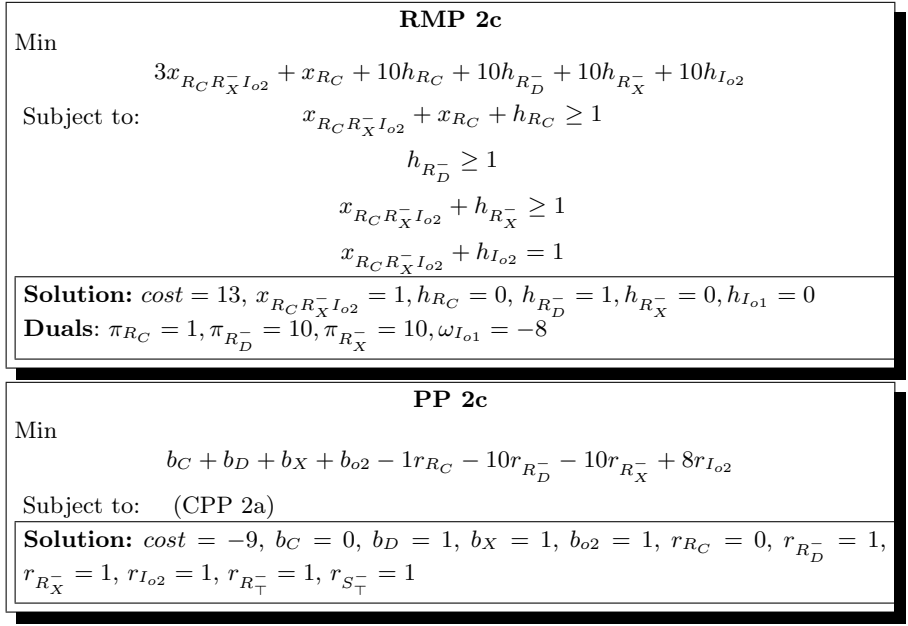

\begin{minipage}[t][1\totalheight][c]{1\columnwidth}%
\begin{flushleft}
{\small{}{}}%
\shadowbox{\begin{minipage}[t][1\totalheight][c]{0.95\columnwidth}%
\begin{center}
\textbf{\small{}{}RMP 2c}{\small{}{}}\\
 {\small{}{}}%
\begin{minipage}[t]{1\columnwidth}%
\begin{flushleft}
\vspace{-3mm}
{\small{}{}Min \setlength{\abovedisplayskip}{0pt} 
\[
3x_{R_{C}R_{X}^{-}I_{o2}}+x_{R_{C}}+10h_{R_{C}}+10h_{R_{D}^{-}}+10h_{R_{X}^{-}}+10h_{I_{o2}}
\]
}\vspace{-4mm}
{\small{}{}Subject to:} 
\par\end{flushleft}%
\end{minipage}{\small{}{}}\\
 {\small{}{}}%
\begin{minipage}[t][1\totalheight][c]{0.99\columnwidth}%
\begin{center}
{\small{}{}$\begin{gathered}x_{R_{C}R_{X}^{-}I_{o2}}+x_{R_{C}}+h_{R_{C}}\geq1\\
h_{R_{D}^{-}}\geq1\\
x_{R_{C}R_{X}^{-}I_{o2}}+h_{R_{X}^{-}}\geq1\\
x_{R_{C}R_{X}^{-}I_{o2}}+h_{I_{o2}}=1
\end{gathered}
$} 
\par\end{center}%
\end{minipage}{\small{}{}}\\
 {\small{}{}}%
\fbox{%
\begin{minipage}[t][1\totalheight][c]{0.99\columnwidth}%
\textbf{\small{}{}Solution:}{\small{}{} $cost=13$, $x_{R_{C}R_{X}^{-}I_{o2}}=1,h_{R_{C}}=0$,
$h_{R_{D}^{-}}=1,h_{R_{X}^{-}}=0,h_{I_{o1}}=0$}{\small \par}

\textbf{\small{}{}Duals}{\small{}{}: $\pi_{R_{C}}=1,\pi_{R_{D}^{-}}=10,\pi_{R_{X}^{-}}=10,\omega_{I_{o1}}=-8$}{\small \par}%
\end{minipage}} 
\par\end{center}%
\end{minipage}}{\small{}{}}%
\newline
\shadowbox{\begin{minipage}[t][1\totalheight][c]{0.95\columnwidth}%
\begin{center}
\textbf{\small{}{}PP 2c}{\small{}{}}\\
 {\small{}{}}%
\begin{minipage}[t]{1\columnwidth}%
\begin{flushleft}
\vspace{-2mm}
{\small{}{}Min \setlength{\abovedisplayskip}{0pt} 
\[
b_{C}+b_{D}+b_{X}+b_{o2}-1r_{R_{C}}-10r_{R_{D}^{-}}-10r_{R_{X}^{-}}+8r_{I_{o2}}
\]
}\vspace{-0mm}
{\small{}{}Subject to:~~~ (CPP 2a)}  
\par\end{flushleft}%
\end{minipage}{\small{}{}}\\
 {\small{}{}}%
\fbox{%
\begin{minipage}[t][1\totalheight][c]{0.99\columnwidth}%
\textbf{\small{}{}Solution:}{\small{}{} $cost=-9$, $b_{C}=0$,
$b_{D}=1$, $b_{X}=1$, $b_{o2}=1$, $r_{R_{C}}=0$, $r_{R_{D}^{-}}=1$,
$r_{R_{X}^{-}}=1$, $r_{I_{o2}}=1$, $r_{R_{\top}^{-}}=1$, $r_{S_{\top}^{-}}=1$}%
\end{minipage}} 
\par\end{center}%
\end{minipage}} 
\par\end{flushleft}%
\end{minipage}{\small{}{} }{\small \par}

\caption{Node $x_{2}$: Third ILP iteration}
\label{fig:ilp-2-c} 
\end{figure}
 
\begin{figure}[tp]
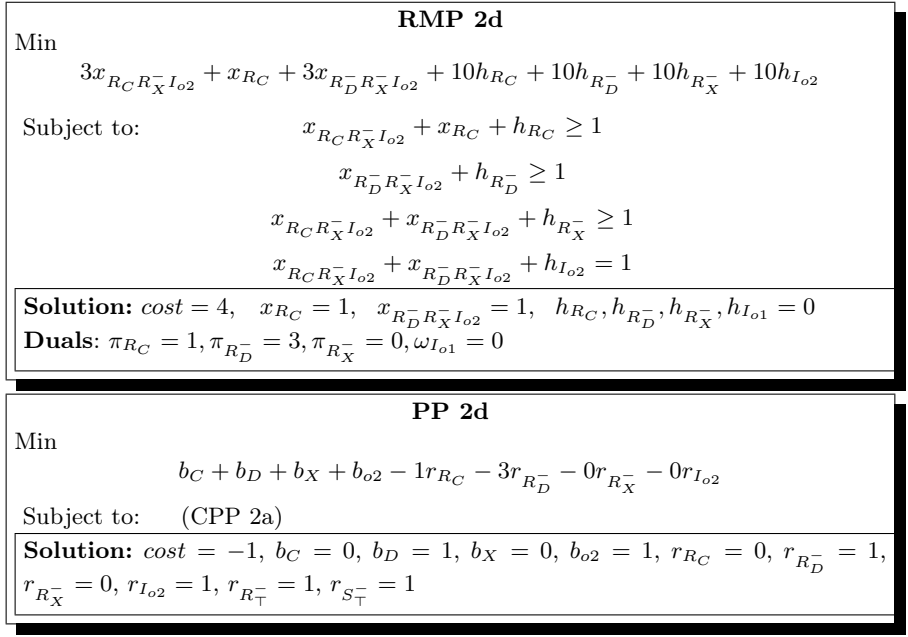

\begin{minipage}[t][1\totalheight][c]{1\columnwidth}%
\begin{flushleft}
{\small{}{}}%
\shadowbox{\begin{minipage}[t][1\totalheight][c]{0.95\columnwidth}%
\begin{center}
\textbf{\small{}{}RMP 2d}{\small{}{}}\\
 {\small{}{}}%
\begin{minipage}[t]{1\columnwidth}%
\begin{flushleft}
\vspace{-3mm}
{\small{}{}Min \setlength{\abovedisplayskip}{0pt} 
\[
3x_{R_{C}R_{X}^{-}I_{o2}}+x_{R_{C}}+3x_{R_{D}^{-}R_{X}^{-}I_{o2}}+10h_{R_{C}}+10h_{R_{D}^{-}}+10h_{R_{X}^{-}}+10h_{I_{o2}}
\]
}\vspace{-4mm}
{\small{}{}Subject to:} 
\par\end{flushleft}%
\end{minipage}{\small{}{}}\\
 {\small{}{}}%
\begin{minipage}[t][1\totalheight][c]{0.99\columnwidth}%
\begin{center}
{\small{}{}$\begin{gathered}x_{R_{C}R_{X}^{-}I_{o2}}+x_{R_{C}}+h_{R_{C}}\geq1\\
x_{R_{D}^{-}R_{X}^{-}I_{o2}}+h_{R_{D}^{-}}\geq1\\
x_{R_{C}R_{X}^{-}I_{o2}}+x_{R_{D}^{-}R_{X}^{-}I_{o2}}+h_{R_{X}^{-}}\geq1\\
x_{R_{C}R_{X}^{-}I_{o2}}+x_{R_{D}^{-}R_{X}^{-}I_{o2}}+h_{I_{o2}}=1
\end{gathered}
$} 
\par\end{center}%
\end{minipage}{\small{}{}}\\
 {\small{}{}}%
\fbox{%
\begin{minipage}[t][1\totalheight][c]{0.99\columnwidth}%
\textbf{\small{}{}Solution:}{\small{}{} $cost=4$, ~ $x_{R_{C}}=1,~~ x_{R_{D}^{-}R_{X}^{-}I_{o2}}=1,~~ h_{R_{C}},
h_{R_{D}^{-}},h_{R_{X}^{-}},h_{I_{o1}}=0$}{\small \par}

\textbf{\small{}{}Duals}{\small{}{}: $\pi_{R_{C}}=1,\pi_{R_{D}^{-}}=3,\pi_{R_{X}^{-}}=0,\omega_{I_{o1}}=0$}{\small \par}%
\end{minipage}} 
\par\end{center}%
\end{minipage}}{\small{}{}}%
\newline
\shadowbox{\begin{minipage}[t][1\totalheight][c]{0.95\columnwidth}%
\begin{center}
\textbf{\small{}{}PP 2d}{\small{}{}}\\
 {\small{}{}}%
\begin{minipage}[t]{1\columnwidth}%
\begin{flushleft}
\vspace{-2mm}
{\small{}{}Min \setlength{\abovedisplayskip}{0pt} 
\[
b_{C}+b_{D}+b_{X}+b_{o2}-1r_{R_{C}}-3r_{R_{D}^{-}}-0r_{R_{X}^{-}}-0r_{I_{o2}}
\]
}\vspace{-0mm}
{\small{}{}Subject to: ~~~ (CPP 2a)}  
\par\end{flushleft}%
\end{minipage}{\small{}{}}\\
 {\small{}{}}%
\fbox{%
\begin{minipage}[t][1\totalheight][c]{0.99\columnwidth}%
\textbf{\small{}{}Solution:}{\small{}{} $cost=-1$, $b_{C}=0$,
$b_{D}=1$, $b_{X}=0$, $b_{o2}=1$, $r_{R_{C}}=0$, $r_{R_{D}^{-}}=1$,
$r_{R_{X}^{-}}=0$, $r_{I_{o2}}=1$, $r_{R_{\top}^{-}}=1$, $r_{S_{\top}^{-}}=1$}%
\end{minipage}} 
\par\end{center}%
\end{minipage}} 
\par\end{flushleft}%
\end{minipage}{\small{}{} }{\small \par}

\caption{Node $x_{2}$: Fourth ILP iteration}
\label{fig:ilp-2-d} 
\end{figure}
 
\begin{figure}[tp]
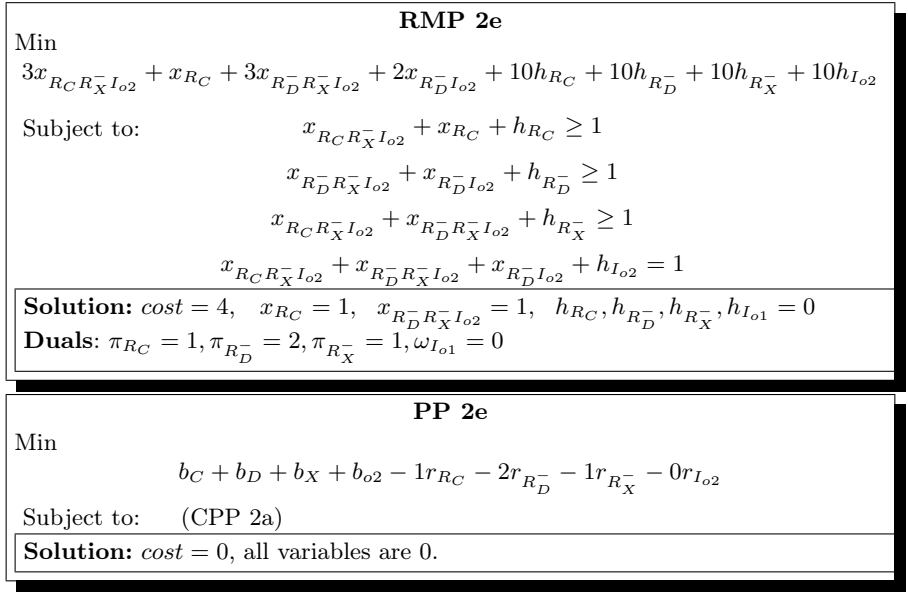

\begin{minipage}[t][1\totalheight][c]{1\columnwidth}%
\begin{flushleft}
{\small{}{}}%
\shadowbox{\begin{minipage}[t][1\totalheight][c]{0.95\columnwidth}%
\begin{center}
\textbf{\small{}{}RMP 2e}{\small{}{}}\\
 {\small{}{}}%
\begin{minipage}[t]{1\columnwidth}%
\begin{flushleft}
\vspace{-3mm}
{\small{}{}Min \setlength{\abovedisplayskip}{0pt} 
\[
3x_{R_{C}R_{X}^{-}I_{o2}}+x_{R_{C}}+3x_{R_{D}^{-}R_{X}^{-}I_{o2}}+2x_{R_{D}^{-}I_{o2}}+10h_{R_{C}}+10h_{R_{D}^{-}}+10h_{R_{X}^{-}}+10h_{I_{o2}}
\]
}\vspace{-4mm}
{\small{}{}Subject to:} 
\par\end{flushleft}%
\end{minipage}{\small{}{}}\\
 {\small{}{}}%
\begin{minipage}[t][1\totalheight][c]{0.99\columnwidth}%
\begin{center}
{\small{}{}$\begin{gathered}x_{R_{C}R_{X}^{-}I_{o2}}+x_{R_{C}}+h_{R_{C}}\geq1\\
x_{R_{D}^{-}R_{X}^{-}I_{o2}}+x_{R_{D}^{-}I_{o2}}+h_{R_{D}^{-}}\geq1\\
x_{R_{C}R_{X}^{-}I_{o2}}+x_{R_{D}^{-}R_{X}^{-}I_{o2}}+h_{R_{X}^{-}}\geq1\\
x_{R_{C}R_{X}^{-}I_{o2}}+x_{R_{D}^{-}R_{X}^{-}I_{o2}}+x_{R_{D}^{-}I_{o2}}+h_{I_{o2}}=1
\end{gathered}
$} 
\par\end{center}%
\end{minipage}{\small{}{}}\\
 {\small{}{}}%
\fbox{%
\begin{minipage}[t][1\totalheight][c]{0.99\columnwidth}%
\textbf{\small{}{}Solution:}{\small{}{} $cost=4$, ~  $x_{R_{C}}=1,~~x_{R_{D}^{-}R_{X}^{-}I_{o2}}=1, ~~ h_{R_{C}},
h_{R_{D}^{-}},h_{R_{X}^{-}},h_{I_{o1}}=0$}{\small \par}

\textbf{\small{}{}Duals}{\small{}{}: $\pi_{R_{C}}=1,\pi_{R_{D}^{-}}=2,\pi_{R_{X}^{-}}=1,\omega_{I_{o1}}=0$}{\small \par}%
\end{minipage}} 
\par\end{center}%
\end{minipage}}{\small{}{}}%
\newline
\shadowbox{\begin{minipage}[t][1\totalheight][c]{0.95\columnwidth}%
\begin{center}
\textbf{\small{}{}PP 2e}{\small{}{}}\\
 {\small{}{}}%
\begin{minipage}[t]{1\columnwidth}%
\begin{flushleft}
\vspace{-2mm}
{\small{}{}Min \setlength{\abovedisplayskip}{0pt} 
\[
b_{C}+b_{D}+b_{X}+b_{o2}-1r_{R_{C}}-2r_{R_{D}^{-}}-1r_{R_{X}^{-}}-0r_{I_{o2}}
\]
}\vspace{-0mm}
{\small{}{}Subject to: ~~~ (CPP 2a)}  
\par\end{flushleft}%
\end{minipage}{\small{}{}}\\
 {\small{}{}}%
\fbox{%
\begin{minipage}[t][1\totalheight][c]{0.99\columnwidth}%
\textbf{\small{}{}Solution:}{\small{}{} $cost=0$, all variables
are 0.}%
\end{minipage}} 
\par\end{center}%
\end{minipage}} 
\par\end{flushleft}%
\end{minipage}{\small{}{} }{\small \par}

\caption{Node $x_{2}$: Fifth ILP iteration}
\label{fig:ilp-2-e} 
\end{figure}

\begin{figure}[tp]
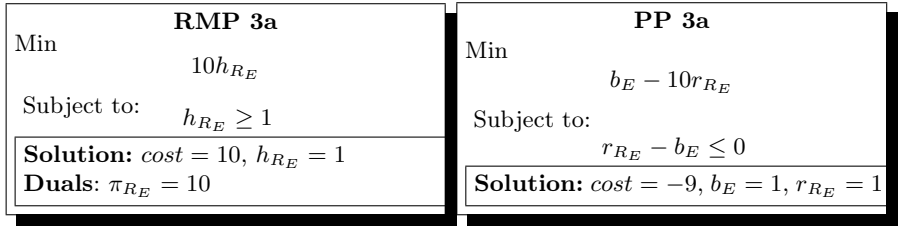

\begin{minipage}[t][1\totalheight][c]{1\columnwidth}%
\begin{flushleft}
{\small{}{}}%
\shadowbox{\begin{minipage}[t][1\totalheight][c]{0.46\columnwidth}%
\begin{center}
\textbf{\small{}{}RMP 3a}{\small{}{}}\\
 {\small{}{}}%
\begin{minipage}[t]{1\columnwidth}%
\begin{flushleft}
\vspace{-3mm}
{\small{}{}Min \setlength{\abovedisplayskip}{0pt} \vspace{-1mm}
\[
10h_{R_{E}}
\]
}\vspace{-2mm}
{\small{}{}Subject to:} 
\par\end{flushleft}%
\end{minipage}{\small{}{}}\\
 {\small{}{}}%
\begin{minipage}[t][1\totalheight][c]{0.99\columnwidth}%
\begin{center}
{\small{}{}$\begin{gathered}h_{R_{E}}\geq1\end{gathered}
$} 
\par\end{center}%
\end{minipage}{\small{}{}}\\
 {\small{}{}}%
\fbox{%
\begin{minipage}[t][1\totalheight][c]{0.99\columnwidth}%
\textbf{\small{}{}Solution:}{\small{}{} $cost=10$, $h_{R_{E}}=1$}{\small \par}

\textbf{\small{}{}Duals}{\small{}{}: $\pi_{R_{E}}=10$}{\small \par}%
\end{minipage}} 
\par\end{center}%
\end{minipage}}{\small{}{}}%
\shadowbox{\begin{minipage}[t][1\totalheight][c]{0.45\columnwidth}%
\begin{center}
\textbf{\small{}{}PP 3a}{\small{}{}}\\
 {\small{}{}}%
\begin{minipage}[t]{1\columnwidth}%
\begin{flushleft}
\vspace{-2mm}
{\small{}{}Min \setlength{\abovedisplayskip}{0pt} 
\[
b_{E}-10r_{R_{E}}
\]
}\vspace{-0mm}
{\small{}{}Subject to:} 
\par\end{flushleft}%
\end{minipage}{\small{}{}}\\
 {\small{}{}}%
\begin{minipage}[t][1\totalheight][c]{1\columnwidth}%
\begin{center}
{\small{}{}$\begin{gathered}r_{R_{E}}-b_{E}\leq0\end{gathered}
$} 
\par\end{center}%
\end{minipage}{\small{}{}}\\
 {\small{}{}}%
\fbox{%
\begin{minipage}[t][1\totalheight][c]{0.99\columnwidth}%
\textbf{\small{}{}Solution:}{\small{}{} $cost=-9$, $b_{E}=1$,
$r_{R_{E}}=1$}%
\end{minipage}} 
\par\end{center}%
\end{minipage}} 
\par\end{flushleft}%
\end{minipage}{\small{}{} }{\small \par}

\caption{Node $x_{3}$: First ILP iteration}
\label{fig:ilp-3-a} 
\end{figure}
 
\begin{figure}[tp]
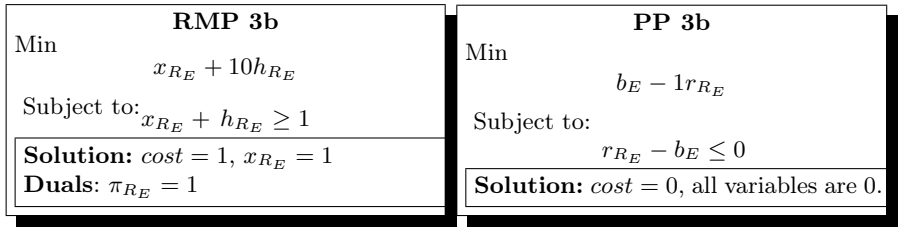

\begin{minipage}[t][1\totalheight][c]{1\columnwidth}%
\begin{flushleft}
{\small{}{}}%
\shadowbox{\begin{minipage}[t][1\totalheight][c]{0.46\columnwidth}%
\begin{center}
\textbf{\small{}{}RMP 3b}{\small{}{}}\\
 {\small{}{}}%
\begin{minipage}[t]{1\columnwidth}%
\begin{flushleft}
\vspace{-3mm}
{\small{}{}Min \setlength{\abovedisplayskip}{0pt} \vspace{-1mm}
\[
x_{R_{E}}+10h_{R_{E}}
\]
}\vspace{-2mm}
{\small{}{}Subject to:} 
\par\end{flushleft}%
\end{minipage}{\small{}{}}\\
 {\small{}{}}%
\begin{minipage}[t][1\totalheight][c]{0.99\columnwidth}%
\begin{center}
{\small{}{}$x_{R_{E}}+\begin{gathered}h_{R_{E}}\geq1\end{gathered}
$} 
\par\end{center}%
\end{minipage}{\small{}{}}\\
 {\small{}{}}%
\fbox{%
\begin{minipage}[t][1\totalheight][c]{0.99\columnwidth}%
\textbf{\small{}{}Solution:}{\small{}{} $cost=1$, $x_{R_{E}}=1$}{\small \par}

\textbf{\small{}{}Duals}{\small{}{}: $\pi_{R_{E}}=1$}{\small \par}%
\end{minipage}} 
\par\end{center}%
\end{minipage}}{\small{}{}}%
\shadowbox{\begin{minipage}[t][1\totalheight][c]{0.45\columnwidth}%
\begin{center}
\textbf{\small{}{}PP 3b}{\small{}{}}\\
 {\small{}{}}%
\begin{minipage}[t]{1\columnwidth}%
\begin{flushleft}
\vspace{-2mm}
{\small{}{}Min \setlength{\abovedisplayskip}{0pt} 
\[
b_{E}-1r_{R_{E}}
\]
}\vspace{-0mm}
{\small{}{}Subject to:} 
\par\end{flushleft}%
\end{minipage}{\small{}{}}\\
 {\small{}{}}%
\begin{minipage}[t][1\totalheight][c]{1\columnwidth}%
\begin{center}
{\small{}{}$\begin{gathered}r_{R_{E}}-b_{E}\leq0\end{gathered}
$} 
\par\end{center}%
\end{minipage}{\small{}{}}\\
 {\small{}{}}%
\fbox{%
\begin{minipage}[t][1\totalheight][c]{0.99\columnwidth}%
\textbf{\small{}{}Solution:}{\small{}{} $cost=0$, all variables
are 0.}%
\end{minipage}} 
\par\end{center}%
\end{minipage}} 
\par\end{flushleft}%
\end{minipage}{\small{}{} }{\small \par}

\caption{Node $x_{3}$: Second ILP iteration}
\label{fig:ilp-3-b} 
\end{figure}

\item By unfolding $C$ in the label of $x_{3}$ we get $\mathcal{L}(x_{3})=\{C,\exists R.E\}$.
RMP starts with artificial variables, $\mathcal{P}'$ is initially
empty, and $Q_{\exists}=\{R_{E}\}$, $Q_{\forall}=\emptyset$ and
$Q_{o}=\emptyset$ (see Fig.\ \ref{fig:ilp-3-a}). 
\item Since the value of $r_{R_{E}}$ is 1 in (PP 3a), the variable $x_{R_{E}}$
is added to (RMP 3b). The cost of RMP is reduced from 10 in (RMP 3a)
to 1 in (RMP 3b). 
\item All artificial variables in (RMP 3b) are zero and the reduced cost
of (PP 3b) is not negative. Therefore, AM terminates after the second
iteration and returns the solution $\sigma(x_{3})=\{\left\langle \left\{ R\right\} ,\left\{ E\right\} ,1\right\rangle \}$. 
\item The $\mathit{fil}$-Rules creates node $x_{4}$ with $\mathcal{L}(x_{4})\leftarrow\{E\}$
and $\card{x_{4}}\leftarrow1$. The $e$-Rule creates edges $\left\langle x_{3},x_{4}\right\rangle $
and $\left\langle x_{4},x_{3}\right\rangle $ with $\mathcal{L}\left(\left\langle x_{3},x_{4}\right\rangle \right)\leftarrow \{R,S\}$
and $\mathcal{L}\left(\left\langle x_{4},x_{3}\right\rangle \right)\leftarrow\{R^{-},S^{-}\}$.
\item $\mathcal{L}(x_{4})=\{E,\forall S^{-}.\left\{ o1\right\} \}$ and
after unfolding $E$ the $\forall$-Rule adds $o1$ to $\mathcal{L}(x_{3})$.
However, $o1$ already occurs in $\mathcal{L}(x_{1})$ and $x_{1}\neq x_{3}$.
Therefore, the $\mathit{nom_{\mathit{merge}}}$-Rule merges node $x_{3}$
into node $x_{1}$. 
\item Since no more rules are applicable, the tableau algorithm terminates.
\end{enumerate}

\section{Proof of Correctness}

In this appendix we present a tableau for DL $\mathcal{SHOI}$ and
proof of the algorithm's termination, soundness and completeness.

\subsection{A Tableau for DL $\mathcal{SHOI}$}
\label{tableau-appendix}

We define a tableau for DL $\mathcal{SHOI}$ based on the standard
tableau for DL $\mathcal{SHOI}Q$, introduced in \cite{horrocks2007tableau}.
For convenience, we assume that all concept descriptions are in negation
normal form, i.e., the negation sign only appears in front of concept names (atomic concepts). 
A label is assigned to each node in
the completion graph (CG) and that label is a subset of possible concept
expressions. We define $clos(C)$ as the closure of a concept expression
$C$.

\noindent \textbf{Definition 1:} The closure $clos(C)$ for a concept
expression $C$ is a smallest set of concepts such that: 
\begin{itemize}
\item $C\in clos(C)$, 
\item $\neg D\in clos(C)\quad\Longrightarrow\quad D\in clos(C)$,
\item $(B\sqcup D)\in clos(C)\:or\:(B\sqcap D)\in clos(C)\quad\Longrightarrow\quad B\in clos(C),D\in clos(C)$,
\item $\forall R.D\in clos(C)\quad\Longrightarrow\quad D\in clos(C)$
\item $\exists R.D\in clos(C)\quad\Longrightarrow\quad D\in clos(C)$
\end{itemize}
where $B,D\in N$ and $R\in N_{R}$. For a Tbox $\mathcal{T}$, if
($C\sqsubseteq D\in\mathcal{T}$) or ($C\equiv D\in\mathcal{T}$)
then $clos(C)\subseteq clos(\mathcal{T})$ and $clos(D)\subseteq clos(\mathcal{T})$.
Similarly, for an Abox $\mathcal{A}$, if ($a:C\in\mathcal{A}$) then
$clos(C)\subseteq clos(\mathcal{A})$.

\noindent \textbf{Definition 2:} If $(\mathcal{T},\mathcal{R})$ is
a $\mathcal{SHOI}$ knowledge base w.r.t. Tbox $\mathcal{T}$ and
role hierarchy $\mathcal{R}$, a tableau $\mathrm{T}$ for $(\mathcal{T},\mathcal{R})$
is defined as a triple $(\mathbf{S},\mathcal{L},\mathcal{E})$ such
that: $\mathbf{S}$ is a set of individuals, $\mathcal{L}:\mathbf{S}\longrightarrow2^{clos(\mathcal{T})}$
maps each individual to a set of concepts, and $\mathcal{E}:N_{R}\longrightarrow2^{\mathbf{S}\times\mathbf{S}}$
maps each role in $N_{R}$ to a set of pairs of individuals in $\mathbf{S}$.
For all $x,y\in\mathbf{S}$, $C,C_{1},C_{2}\in clos(\mathcal{T})$,
and $R,S\in N_{R}$, the following properties must always hold:

\begin{description}
\item [{(P1)}] if $C\in\mathcal{L}(x)$, then $\neg C\notin\mathcal{L}(x)$,
\item [{(P2)}] if $(C_{1}\sqcap C_{2})\in\mathcal{L}(x)$, then $C_{1}\in\mathcal{L}(x)$
and $C_{2}\in\mathcal{L}(x)$, 
\item [{(P3)}] if $(C_{1}\sqcup C_{2})\in\mathcal{L}(x)$, then $C_{1}\in\mathcal{L}(x)$
or $C_{2}\in\mathcal{L}(x)$, 
\item [{(P4)}] if $\forall R.C\in\mathcal{L}(x)$ and $\left\langle x,y\right\rangle \in\mathcal{E}(R)$,
then $C\in\mathcal{L}(y)$, 
\item [{(P5)}] if $\exists R.C\in\mathcal{L}(x)$, then there is some $y\in\mathbf{S}$
such that $\left\langle x,y\right\rangle \in\mathcal{E}(R)$ and $C\in\mathcal{L}(y)$, 
\item [{(P6)}] if $\forall S.C\in\mathcal{L}(x)$ and $\left\langle x,y\right\rangle \in\mathcal{E}(U)$
for some $U,R$ with $U\sqsubseteq_{*}R$, $R\sqsubseteq_{*}S$, and $R\in N_{R_{+}}$, then
$\forall R.C\in\mathcal{L}(y)$, 
\item [{(P7)}] if $\left\langle x,y\right\rangle \in\mathcal{E}(R)$ and
$R\sqsubseteq_{*} S\in\mathcal{R}$, then $\left\langle x,y\right\rangle \in\mathcal{E}(S)$, 
\item [{(P8)}] $\left\langle x,y\right\rangle \in\mathcal{E}(R)$ iff $\left\langle y,x\right\rangle \in\mathcal{E}(\mathsf{Inv}(R))$, 
\item [{(P9)}] if $o\in\mathcal{L}(x)\cap\mathcal{L}(y)$ for some $o\in N_{o}$,
then $x=y$, and
\item [{(P10)}] for each $o\in N_{o}$ occurring in $\mathcal{T}$, $\card{}\{x\in\mathbf{S}\mid o\in\mathcal{L}(x)\}=1$.
\end{description}

\subsection{Proof of the algorithm's termination, soundness and completeness}
\label{sec:proofs}

\noindent These proofs are similar to the one found in \cite{horrocks2007tableau}.
\begin{lemma}
A $\mathcal{SHOI}$ knowledge base $(\mathcal{T},\mathcal{R})$ is
consistent iff there exists a tableau for $(\mathcal{T},\mathcal{R})$. \end{lemma}

\begin{proof}
\noindent The proof is analogous to the one presented in \cite{horrocks2007tableau}.
\textbf{(P9) }and\textbf{ (P10)} ensure that the nominal semantics
are preserved by interpreting them as singletons. \end{proof}

\begin{lemma}
When started with a $\mathcal{SHOI}$ knowledge base $(\mathcal{T},\mathcal{R})$,
the algorithm terminates.\end{lemma}

\begin{proof}
Termination is a consequence of the following properties of the expansion
rules:

\begin{itemize}

\item[1.]Since a partitioning $\mathcal{P}$ is a power set of $Q$
that contains all subsets of $Q$ except the empty set, the size of
$\mathcal{P}$ is bounded by $2^{\card{Q}}-1$ where $\card{Q}$ is the size
of $Q$. $\mathcal{P}$ is computed only once for each node.

\item[2.]Due to the fixed number ($2^{\card{Q}}-1$ ) of variables, the
solution for the linear inequalities can be computed in polynomial
time. Moreover, it does not affect the termination of the expansion
rules. 

\item[3.]Each rule except the $\mathit{nom_\mathit{merge}}$-Rule extends the completion
graph by adding new nodes or extending node labels without removing
nodes or elements from node.

\item[4.]New nodes are only generated by the $\mathit{fil}$-Rule for a concept
$\exists R.C\in L(x)$. The $\mathit{fil}$-Rule can only be triggered once
for each of such concept for a node $x$. If a neighbour $y$ of $x$
was generated by the $\mathit{fil}$-rule for this concept and $y$ is later
merged in some other node $z$ by the $\mathit{nom_\mathit{merge}}$-rule then an
$R$-edge toward $z$ will be created and labels of both nodes will
be merged. Therefore, there will always be some $R$-neighbour of
$x$ with $C$ in its label. Hence, the $\mathit{fil}$-Rule cannot be applied
to $x$ for a concept $\exists R.C$ again.

\item[5.]Equality blocking \cite{horrocks1999description,hladik2004tableau} is used to prevent application of expansion rules when the construction becomes iterative.

\item[6.]The $\mathit{fil}$-Rule can not be applied to the blocked nodes. 

\end{itemize}\end{proof}

\begin{lemma}
\label{lemma-3}
If the expansion rules can be applied to a $\mathcal{SHOI}$ knowledge
base $(\mathcal{T},\mathcal{R})$ in such a way that they yield a
complete and clash-free completion graph, then there exists a tableau
for $(\mathcal{T},\mathcal{R})$ .\end{lemma}

\begin{proof}
Let $G=(V,E,\mathcal{L})$ be a complete and clash-free
completion graph. A tableau $\mathrm{T}=(\mathbf{S},\mathcal{L}',\mathcal{E})$
for $(\mathcal{T},\mathcal{R})$ can be obtained from $G$ such that:
$\mathbf{S}=V$, $\mathcal{L}'(x)=\mathcal{L}(x)$, and $\mathcal{E}(R)=\{\left\langle x,y\right\rangle \in E\mid\left(\{R\}\cap\mathcal{L}(\left\langle x,y\right\rangle )\right)\neq\emptyset\}$.
In order to show that $\mathrm{T}$ is a tableau for $(\mathcal{T},\mathcal{R})$
, we prove that $\mathrm{T}$ satisfies all properties (\textbf{P1})
- (\textbf{P10}) of tableau (see Definition 9):

\begin{itemize}

\item Since $G$ is clash-free, (\textbf{P1}) holds for $\mathrm{T}$.

\item Let $x$ be an individual in $\mathbf{S}$ with $C_{1}\sqcap C_{2}\in\mathcal{L}'(x)$,
$C_{1}\in\mathcal{L}'(x)$ and $C_{2}\notin\mathcal{L}'(x)$
and having a corresponding node $x$ in $G$ with $C_{1}\sqcap C_{2}\in\mathcal{L}(x)$,
$C_{1}\in\mathcal{L}(x)$ and $C_{2}\notin\mathcal{L}(x)$. It would make
the $\sqcap$-Rule applicable to node $x$ in $G$ but that is not possible
because $G$ is complete. Hence, (\textbf{P2}) holds for $\mathrm{T}$
and likewise (\textbf{P3}).

\item Consider $\forall R.C\in\mathcal{L}'(x)$ and $\langle x,y\rangle \in\mathcal{E}(R)$.
According to the definition of $\mathcal{E}(R)$, $\left\langle x,y\right\rangle \in\mathcal{E}(R)$
implies $\{\langle x,y\rangle \in E \mid \mathcal{L}(\langle x,y \rangle) \cap\{R\}\} \neq \emptyset$.
Since $G$ is complete, it implies that $C\in\mathcal{L}(y)$ and
consequently $C\in\mathcal{L}'(y)$. Therefore, (\textbf{P4})
is satisfied. 

\item For (\textbf{P5}), consider $\exists R.C\in\mathcal{L}'(x)$ and there exists no $R$-neighbour $y\in\mathbf{S}$ of $x$ with $C \subseteq \mathcal{L}'(y)$. It would make the $\mathit{fil}$-Rule
applicable to node $x$ in $G$ but that is not possible
because $G$ is complete. Therefore, (\textbf{P5}) holds for $\mathrm{T}$. 

\item (\textbf{P6}) is similar to (\textbf{P4}). 

\item (\textbf{P7}) and (\textbf{P8}) hold due to the definition
of $R$-neighbour. Furthermore, for (\textbf{P8}),
consider $\left\langle x,y\right\rangle \in\mathcal{E}(R)$ in $\mathrm{T}$
that implies $\{{\langle x,y\rangle \in E} \mid \mathcal{L}(\langle x,y\rangle) \cap \{R\}\} \neq \emptyset$
in $G$. Since $G$ is complete, it implies that $\{{\langle y,x\rangle \in E} \mid \mathcal{L}(\langle y,x\rangle) \cap \{\mathsf{Inv}(R)\}\} \neq \emptyset$
(due to the $e$-Rule) and consequently $\left\langle y,x\right\rangle \in\mathcal{E}(\mathsf{Inv}(R))$.

\item For (\textbf{P9}), consider $o\in\mathcal{L}'(x)\cap\mathcal{L}'(y)$
where $x\neq y$, and having two corresponding distinct nodes $x$
and $y$ in $G$ with $o\in\mathcal{L}(x)\cap\mathcal{L}(y)$ for
some nominal $o\in N_{o}$. In this case, the $\mathit{nom_\mathit{merge}}$-Rule
would need to be fired but that is not possible because $G$ is complete. Hence,
(\textbf{P9}) holds for $\mathrm{T}$. 

\item Since the partition elements of $\mathcal{P}$ are semantically pair-wise
 disjoint, i.e., if $p,p' \in \mathcal{P}$, $p\neq p'$ then $p^{\mathcal{I}} \cap (p')^{\mathcal{I}} = \emptyset$, and due to the nominal semantics $\#\{o\}^{I}=1$, the algebraic reasoner assigns the nominal $o$ to only one partition element $p$.
Therefore, the cardinality of $p$ will always be 1. In addition,
since nominals always exist, the nominal nodes are never removed through
pruning. Hence, (\textbf{P10}) always holds. 

\end{itemize}\end{proof}

\begin{lemma}
If there exists a tableau of a $\mathcal{SHOI}$ knowledge base $(\mathcal{T},\mathcal{R})$,
then, the expansion rules can be applied to a $\mathcal{SHOI}$ knowledge
base $(\mathcal{T},\mathcal{R})$ in such a way that they yield a
complete and clash-free completion graph.\end{lemma}

\begin{proof}
Let $\mathrm{T}=(\mathbf{S},\mathcal{L}',\mathcal{E})$ be
a tableau for $(\mathcal{T},\mathcal{R})$. We use $\mathrm{T}$ to
trigger the application of the expansion rules such that they yield
a complete and clash-free completion graph $G=(V,E,\mathcal{L})$.
A function $\pi$ is used to map the nodes of G to elements of $\mathbf{S}$.
For all $x,y\in V$ and $R\in N_{R}$, the mapping $\pi$ satisfies
the following properties:

\begin{itemize}

\item[1.] $\mathcal{L}(x)\subseteq\mathcal{L}'(\pi(x))$

\item[2.] if $\left\langle x,y\right\rangle \in E$ and $R\in\mathcal{L}(\langle x,y\rangle) $
then $\left\langle \pi(x),\pi(y)\right\rangle \in\mathcal{E}(R)$

\item[3.] $x\neq y$ implies $\pi(x)\neq\pi(y)$

\end{itemize}

We show by applying the expansion rules defined in Figure \ref{Exp_Rules}
in order to obtain $G$, the properties of mapping $\pi$ are not
violated:

\begin{itemize}

\item The $\sqcap$-Rule, $\sqcup$-Rule, the $\forall$-Rule and
the $\forall_{+}$-Rule extend the label of a node $x$ without violating
$\pi$ properties due to (\textbf{P2}) - (\textbf{P4}) and (\textbf{P6})
of tableau. 

\item \textbf{The $\mathit{nom_\mathit{merge}}$-Rule:} If $o\in\mathcal{L}(x)\cap\mathcal{L}(y)$
for some nominal $o\in N_{o}$, then $o\in\mathcal{L}'(\pi(x))\cap\mathcal{L}'(\pi(y))$.
Since $\mathrm{T}$ is a tableau, (\textbf{P9}) and (\textbf{P10}) imply $\pi(x)=\pi(y)$.
Therefore, the $\mathit{nom_\mathit{merge}}$-Rule can be applied to merge nodes $x$
and $y$ without violating $\pi$ properties. 

\item \textbf{The $\mathit{inverse}$-Rule:} Since $\mathrm{T}$ is a tableau,
(\textbf{P9}) implies that if $\left\langle \pi(x),\pi(y)\right\rangle \in\mathcal{E}(R)$,
then there is a back edge $\left\langle \pi(y),\pi(x)\right\rangle \in\mathcal{E}(\mathsf{Inv}(R))$.
The $\mathit{inverse}$-Rule adds a tuple $\left\langle x,\{R^{-}\}\right\rangle $ to $\mathcal{B}(y)$ if $\forall R^{-}.C\in\mathcal{L}(y)$.
Since the $\mathit{inverse}$-Rule only adds information about an edge that already exists and this information is only used by AM, it does not violate $\pi$
properties.

\item \textbf{The $\mathit{fil}$-Rule:} The numerical restrictions imposed
by existential restrictions and nominals are encoded into inequalities.
The solution $\sigma$ is returned by the algebraic reasoner that
defines the distribution of fillers by satisfying the inequalities.
The $\mathit{fil}$-Rule creates a node for each corresponding partition element
returned by the algebraic reasoner. Since $\mathrm{T}$ is a tableau,
(\textbf{P5}) implies that there exist the individuals of $\mathbf{S}$ satisfying these existential
restrictions. Therefore, the $\mathit{fil}$-Rule does not violate properties
of tableau or $\pi$.

\item \textbf{The $e$-Rule:} For each $\exists R.C\in\mathcal{L}(x)$
we have $\exists R.C\in\mathcal{L}'(\pi(x))$ that means there
exists $\pi(y)\in\mathbf{S}$, $C\in\mathcal{L}'(\pi(y))$
and $\left\langle \pi(x),\pi(y)\right\rangle \in\mathcal{E}(R)$.
Since $\mathrm{T}$ is a tableau, (\textbf{P7}) implies $\left\langle \pi(x),\pi(y)\right\rangle \in\mathcal{E}(S)$ if 
$R\sqsubseteq_{*} S\in\mathcal{R}$ and (\textbf{P8}) implies $\left\langle \pi(y),\pi(x)\right\rangle \in\mathcal{E}(\mathsf{Inv}(R))$.
The $e$-Rule is applied to connect $x$ to its fillers, say $y_{i}$,
by 
creating edges $\left\langle x,y_{i}\right\rangle \in E$ between them, 
and to merge $\mathsf{R}$ into $\mathcal{L}(\left\langle x,y_{i}\right\rangle)$ and $\{\mathsf{Inv}(R) \mid R \in \mathsf{R}\}$ into $\mathcal{L}(\left\langle y_{i}, x \right\rangle)$ and all $S$ with $R\sqsubseteq_{*}S\in\mathcal{R}$ into $\mathcal{L}(\left\langle x,y_{i} \right\rangle)$ and $\mathsf{Inv}(S)$ into $\mathcal{L}(\left\langle y_{i},x \right\rangle)$ without violating $\pi$ properties. 
Moreover, (\textbf{P7}) and (\textbf{P8})
of tableau are also preserved. 

\end{itemize}\end{proof}

\end{document}